\documentclass{article} 
\usepackage{iclr2027_conference,times}

\usepackage[T1]{fontenc}
\usepackage[utf8]{inputenc}
\usepackage{placeins}
\usepackage{amsmath,amssymb,amsfonts,amsthm,mathtools}
\usepackage{bm,bbm,mathrsfs}
\usepackage{microtype}
\usepackage{xcolor}
\usepackage{graphicx}
\usepackage{booktabs,tabularx,array,makecell,multirow}
\usepackage{enumitem}
\usepackage{algorithm,algorithmic}
\usepackage{caption,subcaption}
\usepackage{wrapfig}
\usepackage{adjustbox}
\usepackage{nicefrac}
\usepackage{siunitx}
\usepackage{tikz}
\usetikzlibrary{positioning,arrows.meta,calc,matrix}
\usepackage{pgfplots}
\usepackage{pgfplotstable}
\pgfplotsset{compat=1.18}

\usepackage{hyperref}
\hypersetup{colorlinks=true,linkcolor=blue!50!black,citecolor=blue!50!black,urlcolor=blue!50!black}
\usepackage[nameinlink,capitalise,noabbrev]{cleveref}
\usepackage{titletoc}

\setlist[itemize]{topsep=3pt,itemsep=2pt,parsep=0pt,leftmargin=1.5em}
\setlist[enumerate]{topsep=3pt,itemsep=2pt,parsep=0pt,leftmargin=1.7em}

\newcolumntype{P}[1]{>{\raggedright\arraybackslash}p{#1}}
\newcolumntype{Y}{>{\raggedright\arraybackslash}X}

\makeatletter
\@ifundefined{theorem}{%
  \theoremstyle{plain}
  \newtheorem{theorem}{Theorem}[section]
}{}
\@ifundefined{proposition}{\newtheorem{proposition}[theorem]{Proposition}}{}
\@ifundefined{lemma}{\newtheorem{lemma}[theorem]{Lemma}}{}
\@ifundefined{corollary}{\newtheorem{corollary}[theorem]{Corollary}}{}
\@ifundefined{assumption}{%
  \theoremstyle{definition}
  
}{}
\@ifundefined{definition}{%
  \theoremstyle{definition}
  
}{}
\@ifundefined{example}{%
  \theoremstyle{definition}
  
}{}
\@ifundefined{remark}{%
  \theoremstyle{remark}
  
}{}
\@ifundefined{condition}{%
  \theoremstyle{definition}
  
}{}
\makeatother

\crefname{assumption}{Assumption}{Assumptions}
\Crefname{assumption}{Assumption}{Assumptions}

\providecommand{\R}{\mathbb{R}}
\providecommand{\bR}{\mathbb{R}}

\providecommand{\E}{\mathbb{E}}

\providecommand{\1}{\mathbbm{1}}

\providecommand{\cB}{\mathcal{B}}

\providecommand{\cD}{\mathcal{D}}

\providecommand{\cL}{\mathcal{L}}

\providecommand{\cN}{\mathcal{N}}

\providecommand{\cT}{\mathcal{T}}

\providecommand{\Var}{\operatorname{Var}}
\providecommand{\Cov}{\operatorname{Cov}}

\providecommand{\Law}{\operatorname{Law}}

\providecommand{\diag}{\operatorname{diag}}

\providecommand{\tr}{\operatorname{tr}}

\makeatletter
\@ifundefined{ceil}{}{}
\@ifundefined{floor}{}{}
\@ifundefined{paren}{}{}
\@ifundefined{bracket}{}{}
\@ifundefined{set}{}{}
\@ifundefined{abs}{}{}
\@ifundefined{norm}{\DeclarePairedDelimiter{\norm}{\lVert}{\rVert}}{}
\@ifundefined{ip}{\DeclarePairedDelimiterX{\ip}[2]{\langle}{\rangle}{#1,\,#2}}{}
\@ifundefined{gendivx}{\DeclarePairedDelimiterX{\gendivx}[2]{(}{)}{#1\;\delimsize\|\;#2}}{}
\makeatother

\newcommand{\papertablestyle}{\small\setlength{\tabcolsep}{4pt}\renewcommand{\arraystretch}{1.12}}

\usepackage{hyperref}
\usepackage{url}

\title{EGGROLL, Unrolled: \\ Understanding and Improving Low-Rank \\ Evolution Strategies at Scale}

\author{Ege~C.~Kaya,  Abolfazl~Hashemi \\
Elmore Family School of Electrical and Computer Engineering\\
Purdue University\\
West Lafayette, IN 47907, USA \\
\texttt{\{kayae,abolfazl\}@purdue.edu}
}

\iclrfinalcopy 
\begin{document}

\maketitle

\begin{abstract}
EGGROLL \citep{sarkar2026evolution} makes evolution strategies (ES) practical for LLMs by replacing dense Gaussian weight perturbations with low-rank Gaussian products, often of rank one. This choice is computationally attractive but geometrically severe: Each rank-one perturbation lies in a zero-volume subset of the ambient matrix space, despite having identity covariance. We characterize the EGGROLL update mean field at finite rank and nonzero perturbation radius as a resolvent applied to the gradient of the perturbation-smoothed objective. This transformation can make the mean field nonconservative and reverse the local stability of an optimum. EGGROLL nevertheless recovers the gradient exactly on quadratic objectives at every rank and radius. In finite populations, the additional sampling variance of rank-one perturbations relative to dense Gaussian ES decays inversely with matrix width under a local affine model, and is only $0.098\%$ at width $4096$. Finally, we introduce \emph{LOO-ROLL}, a leave-one-out estimator that replaces EGGROLL's two antithetic evaluations per direction by one. At equal evaluation cost, LOO-ROLL halves estimator MSE in transformer blocks. Across fourteen post-training settings up to 14B parameters, matched-time comparisons with EGGROLL yield eleven improvements in individual paired tests and no significant loss. At 1.7B, 8B, and 14B parameters, matched-time gains are $2.9$, $14.1$, and $7.9$ percentage points on GSM8K and $12.2$, $8.4$, and $8.1$ points on MATH-500.
\end{abstract}

\section{Introduction}
\label{sec:intro}

Evolution strategies (ES) \citep{schwefel1977numerische,rechenberg1978evolutionsstrategien,beyer2002evolution,salimans2017evolution} optimize a scalar fitness using perturbed function evaluations. Their forward-only computation has motivated memory-efficient language-model fine-tuning \citep{malladi2023fine,pmlr-v235-gautam24a,zhao2025second}, while their parallel populations support scalable post-training \citep{sarkar2026evolution,hoy2026matching,xu2026quantized}. A remaining difficulty is moving a large population of independently perturbed parameter settings through the network efficiently. Modern accelerators obtain high throughput when arithmetic reuses data already brought into fast memory \citep{williams2009roofline,dao2022flashattention}. Giving each member a dense perturbation requires generating or reading another $mn$ numbers per weight matrix, making memory traffic grow with the population.

EGGROLL \citep{sarkar2026evolution} perturbs each weight matrix with a rank-$r$ Gaussian product,
\begin{equation}\label{eq:perturbation}
E_r = \frac{1}{\sqrt r}AB^\top = \frac{1}{\sqrt r}\sum_{s=1}^r a_s b_s^\top, \qquad a_s\sim\mathcal N(0,I_m),\quad b_s\sim\mathcal N(0,I_n),
\end{equation}
which permits the evaluation of very large populations at near-inference throughput. The normalization by $1/\sqrt r$ gives $\mathbb E[\operatorname{vec}(E_r)\operatorname{vec}(E_r)^\top]=I_{mn}$ like in dense Gaussian ES, yet each realization has rank at most $r$. When $r<\min(m,n)$, the possible realizations occupy zero volume in the ambient matrix space \citep{absil2008optimization}. For an activation $x$, the perturbed computation is
\begin{equation}\label{eq:adapter}
(W+\sigma E_r)x=Wx+\frac{\sigma}{\sqrt r}A(B^\top x).
\end{equation}
The population shares the base weights $W$, while each member receives a correction through two thin multiplications with $A$ and $B^\top$. These require $O(r(m+n))$ additional storage and arithmetic rather than the $O(mn)$ of dense Gaussian ES. Low-rank adapters are already widely used in LLM training \citep{hu2022lora,dettmers2023qlora,zhang2023adaptive,liu2024dora}. EGGROLL adopts this representation to estimate an update to network weights without backpropagation.

EGGROLL performs well in LLM post-training even in the seemingly prohibitive regime of $r=1$. The original analysis provides asymptotic guarantees, but does not identify the field at the finite ranks and radii used in practice. We ask: \emph{(i) what mean field does EGGROLL follow, (ii) how accurately does a finite population estimate it, and (iii) what improvements does our analysis suggest?}

Our answers reveal the mean dynamics of EGGROLL and the sampling accuracy of a finite population. The mean field is a resolvent-filtered gradient, which can be nonconservative and can make a local optimum repelling. The discrepancy vanishes on quadratics, and the first local finite-rank correction is proportional to $\sigma^2/r$. The sampling cost is milder: rank-$r$ perturbations add $2(m+n+1)/r$ to the dense-Gaussian variance factor $mn+1$. Rank one therefore retains nearly all of the sampling efficiency of dense Gaussian ES for wide matrices. We then introduce \emph{LOO-ROLL}. \href{https://github.com/ESHyperscale/HyperscaleES}{The released EGGROLL implementation} uses two antithetic fitness evaluations per direction. LOO-ROLL uses the other population members as leave-one-out baselines, requiring one evaluation per direction while preserving EGGROLL's expected update. The saving can then be used on additional directions or optimization steps. Across fourteen settings, two model families, and models up to 14B, eleven matched-time gains are individually significant and nine survive Holm-Bonferroni correction \citep{holm1979simple}. The analysis also suggests rank extrapolation to mitigate the bias introduced by finite rank and control variates to mitigate the high variance caused by ES. However, in tested LLM runs, the additional cost of these modifications do not justify their usage, hence we treat them in \cref{app:secondaryremedies}.

\noindent\textbf{We summarize our contributions as follows.}
\begin{itemize}[leftmargin=1.4em,itemsep=1pt,topsep=2pt]
\item We derive the exact finite-rank EGGROLL population field as a resolvent-filtered gradient and show that it can be nonconservative and locally destabilize an optimum.
\item We prove exactness on quadratics, identify the leading finite-rank correction, and show that the surplus finite-population variance relative to dense Gaussian ES decays with matrix width.
\item We introduce LOO-ROLL, replacing antithetic pairs with one-evaluation leave-one-out directions while preserving the mean EGGROLL field. Across fourteen LLM post-training settings up to 14B parameters, eleven matched-time gains are individually significant, nine survive Holm--Bonferroni correction, and none significantly favors EGGROLL.
\end{itemize}

\noindent\textbf{Related work.}\label{sec:related}
Random-direction finite differences and simultaneous perturbation are classical zeroth-order methods \citep{spall1992multivariate,nesterov2017random,duchi2015optimal}. MeZO fine-tunes LLMs with inference-level memory, and subsequent work adds variance reduction \citep{malladi2023fine,pmlr-v235-gautam24a}. Baselines reduce variance in score-function estimators \citep{williams1992simple,wierstra2014natural}, with REINFORCE leave-one-out applied to LLM response groups \citep{ahmadian2024back} and curvature-aware zeroth-order tuning \citep{seung2026low}. LOO-ROLL applies this principle to EGGROLL's Gaussian-product population. Scalar product-normal Stein identities are established \citep{gaunt2017stein,gaunt2018products}. Our matrix operator retains the shared row and column factors and connects their geometry to optimization dynamics. Extended comparisons appear in \cref{app:related}.

\section{The finite-rank mean field}\label{sec:law}\label{sec:setup}
We start our inquiry by asking what finite-rank EGGROLL updates follow in expectation. We consider one parameter matrix $W\in\mathbb R^{m\times n}$. Appendix~\ref{app:network} combines independent perturbations across the matrix blocks of a network. Let $f:\mathbb R^{m\times n}\to\mathbb R$ be the fitness to maximize, and use the Frobenius inner product $\langle A,B\rangle_F=\operatorname{tr}(A^\top B)$. Sampling perturbed weights produces the smoothed objective
\begin{equation}\label{eq:smoothed}
F_{r,\sigma}(W)=\mathbb E[f(W+\sigma E_r)],
\end{equation}
where $\sigma>0$ is the perturbation radius. EGGROLL weights each perturbation by its observed fitness. Averaging this update gives the mean field
\begin{equation}\label{eq:field}
g_{r,\sigma}^f(W)=\frac1\sigma\mathbb E[E_rf(W+\sigma E_r)].
\end{equation}
This is the infinite-population update that we will later approximate with a finite population. We write $E_\infty$ for a dense standard Gaussian matrix, corresponding to the $r\to\infty$ limit of the Gaussian-product law. For a perturbation density $p$, the \emph{score} is $\nabla\log p$. A dense standard Gaussian matrix $E_\infty$ has score $-E_\infty$. Under the usual regularity conditions, Stein's Gaussian identity \citep{stein1981estimation} gives
\begin{equation}
\frac1\sigma\mathbb E[E_\infty f(W+\sigma E_\infty)]
=\mathbb E[\nabla f(W+\sigma E_\infty)]=\nabla F_{\infty,\sigma}(W).
\end{equation}
The same identity does not automatically hold for the Gaussian-product law, since at finite rank, $-E_r$ is not its exact score.

The reported experiments, fitness shaping, and released LLM implementation of \citet{sarkar2026evolution} use paired antithetic perturbations $+E_r$ and $-E_r$:
\begin{equation}\label{eq:antithetic}
\widehat g_{r,\sigma,N}(W)=\frac1N\sum_{s=1}^N E_r^{(s)}
\frac{f(W+\sigma E_r^{(s)})-f(W-\sigma E_r^{(s)})}{2\sigma}.
\end{equation}
Here $N$ is the number of independent directions, each evaluated with both signs. The symmetry of $E_r$ makes $\widehat g_{r,\sigma,N}$ an unbiased estimator of the mean field. Evaluating both signs on the same prompt batch also isolates the effect of the perturbation. Any fitness component shared by the pair, such as the difficulty of the prompts in that batch, cancels when the two evaluations are subtracted. For a smooth objective, a Taylor expansion around $W$ makes this explicit:
\begin{equation}\label{eq:centraldifference}
\frac{f(W+\sigma E_r)-f(W-\sigma E_r)}{2\sigma}
=\langle\nabla f(W),E_r\rangle_F+O(\sigma^2\|E_r\|_F^3),
\end{equation}
where the constant and other even-order terms cancel. EGGROLL normally also centers and standardizes fitness values across the population. Our main analysis studies the raw, unstandardized field in \cref{eq:field}. \Cref{app:normalization} treats the standardization and quantifies its difference from the raw field.

The definitions above lead to two deterministic questions. First, how does the EGGROLL mean field $g_{r,\sigma}^f$ differ from the gradient $\nabla F_{r,\sigma}$ of the objective smoothed by the same perturbations? Second, when they differ, is $g_{r,\sigma}^f$ still \emph{conservative}, meaning that it is the gradient of some scalar objective? Sampling with $E_r$ smooths $f$ into $F_{r,\sigma}$, while weighting those samples by the dense-Gaussian score produces the mean field. We first characterize the discrepancy between these two fields and then determine when the EGGROLL field remains conservative.

Both $F_{r,\sigma}$ and $g_{r,\sigma}^f$ can be written as convolutions. The smoothed objective $F_{r,\sigma}$ convolves $f$ with the Gaussian-product law, while each component of the mean field $g_{r,\sigma}^f$ convolves $f$ with the same law weighted by the corresponding entry of $E_r/\sigma$. Hence, a Fourier frequency analysis lets us compare the smoothed gradient and the EGGROLL mean field. Letting $\Phi_r(T)=\mathbb E[e^{i\langle T,E_r\rangle_F}]$ denote the characteristic function of the perturbation law, we have the following result.
\begin{theorem}[Exact Gaussian-product law]\label{thm:law}
For $T\in\mathbb R^{m\times n}$,
\begin{equation}\label{eq:char}
\Phi_r(T)=\det\left(I_m+\frac{TT^\top}{r}\right)^{-r/2}, \qquad \nabla_T\Phi_r(T)=-\Phi_r(T)\left(I_m+\frac{TT^\top}{r}\right)^{-1}T.
\end{equation}
\end{theorem}
The proof is given in \cref{app:field}.
For the Fourier mode $f_T(W)=e^{i\langle T,W\rangle_F}$, this gives
\begin{equation}\label{eq:multiplier}
g_{r,\sigma}^{f_T}(W)=i\Phi_r(\sigma T)\left(I_m+\frac{\sigma^2TT^\top}{r}\right)^{-1}T e^{i\langle T,W\rangle_F}.
\end{equation}
For the same mode, direct substitution into \cref{eq:smoothed} gives
\begin{equation}
F_{r,\sigma}(W)=\Phi_r(\sigma T)e^{i\langle T,W\rangle_F},\qquad \nabla F_{r,\sigma}(W)=i\Phi_r(\sigma T)T e^{i\langle T,W\rangle_F}.
\end{equation}
Comparing this expression with \cref{eq:multiplier}, smoothing contributes the common scalar factor $\Phi_r(\sigma T)$, while the score mismatch tilts the gradient direction $T$ to
\begin{equation}\label{eq:Jmode}
J(T)=\left(I_m+\frac{\sigma^2TT^\top}{r}\right)^{-1}T.
\end{equation}
\Cref{fig:lawgeometry} visualizes this transformation on the diagonal subspace $T=\operatorname{diag}(t_1,t_2)$.
\begin{figure}[t]
\centering
\includegraphics[width=\textwidth]{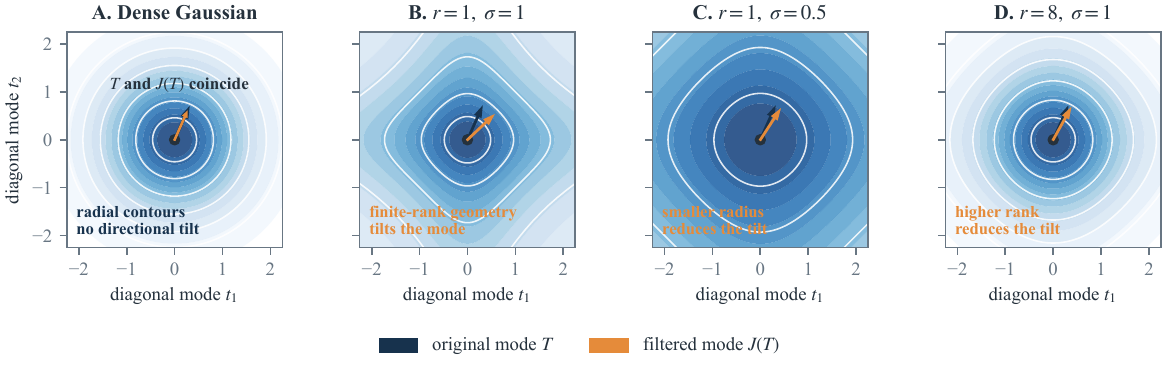}
\caption{Geometry on the diagonal subspace $T=\operatorname{diag}(t_1,t_2)$. Darker contours denote weaker attenuation. Blue arrow gives the gradient-mode direction $T$, and orange arrow gives the population-mode direction $J(T)$. \textbf{A:} A dense Gaussian has radial contours and does not tilt the mode. \textbf{B:} The rank-one product law is nonradial and tilts $J(T)$ away from $T$. \textbf{C:} Decreasing $\sigma$ reduces attenuation and tilt. \textbf{D:} Increasing rank moves the geometry toward the dense-Gaussian limit.}
\label{fig:lawgeometry}
\end{figure}

To translate the frequency-wise multiplier back into weight space, define the differential operator
\begin{equation}\label{eq:Ldef}
(\mathcal L V)_{ij}=-\sum_{k=1}^m\sum_{\ell=1}^n\partial_{i\ell}\partial_{k\ell}V_{kj},
\end{equation}
where $\partial_{ab}$ differentiates with respect to $W_{ab}$. On a mode $Ce^{i\langle T,W\rangle_F}$, $\mathcal L$ multiplies $C$ by $TT^\top$. Combining these frequency-wise transformations recovers the relationship between the complete mean field and smoothed gradient.
\begin{theorem}[The finite-rank resolvent]\label{thm:resolvent}
If $f$ is locally integrable and of at most polynomial growth,
\begin{equation}\label{eq:resolvent}
g_{r,\sigma}^f=\left(I+\frac{\sigma^2}{r}\mathcal L\right)^{-1}\nabla F_{r,\sigma}.
\end{equation}
The equality holds in the sense of tempered distributions. It holds pointwise if $f(W)=\int e^{i\langle T,W\rangle_F}\mu_f(dT)$ with $\int(1+\|T\|_F^3)|\mu_f|(dT)<\infty$.
\end{theorem}
The inverse $(I+\sigma^2\mathcal L/r)^{-1}$ is the \emph{resolvent} of the positive operator $\sigma^2\mathcal L/r$ \citep{bauschke2020correction}. Thus, the mean field is obtained by spectrally damping the smoothed gradient rather than merely rescaling it. The distributional statement covers polynomially growing objectives such as quadratics, while the stronger Fourier condition gives an ordinary pointwise equality. \Cref{app:field} gives the proof and regularity details. If $T$ has singular values $s_j$, the resolvent scales its corresponding directions by $(1+\sigma^2s_j^2/r)^{-1}$. Larger singular values are thus attenuated more strongly, and the unequal factors can rotate the mode direction despite identity covariance.

The resolvent identity also quantifies how far the EGGROLL mean field can depart from the gradient of the correspondingly smoothed objective.
\begin{corollary}[Exact score-mismatch control]\label{cor:mismatch}
Let $J_{r,\sigma}=(I+\sigma^2\mathcal L/r)^{-1}$. If $\nabla F_{r,\sigma}$ belongs to the domain of $\mathcal L$, then
\begin{equation}\label{eq:mismatchbound}
g_{r,\sigma}^f-\nabla F_{r,\sigma}=-\frac{\sigma^2}{r}J_{r,\sigma}\mathcal L\nabla F_{r,\sigma},\qquad
\|g_{r,\sigma}^f-\nabla F_{r,\sigma}\|_{L^2}\le\frac{\sigma^2}{r}\|\mathcal L\nabla F_{r,\sigma}\|_{L^2}.
\end{equation}
\end{corollary}
Thus the departure from the smoothed gradient is controlled by its variation in the matrix-coupled directions measured by $\mathcal L$, and vanishes linearly with $\sigma^2/r$. The proof appears in \cref{app:field}.

\section{When does finite rank change the dynamics?}\label{sec:conservativity}\label{sec:benign}
The resolvent characterizes the EGGROLL mean field, but does not guarantee that it optimizes any scalar objective. We therefore ask which perturbation laws always produce a conservative field, equivalently, one with a symmetric Jacobian \citep{spivak1965calculus}. Real trigonometric polynomials suffice as a test class because their finite Fourier sums can isolate arbitrary frequencies. Failure on this class rules out a universal guarantee for any broader class containing it.
\begin{theorem}[Universal conservativity]\label{thm:conservative}
Let $d\ge2$ and let $Z\in\mathbb R^d$ be centered with $\mathbb E[\|Z\|]<\infty$. The field
\begin{equation}
x\mapsto \frac1\sigma\mathbb E[Zf(x+\sigma Z)]
\end{equation}
is conservative for every real trigonometric polynomial $f$ and every $\sigma>0$ if and only if the law of $Z$ is spherically symmetric.
\end{theorem}

Spherical symmetry makes the characteristic function $\Phi(t)$ radial, so $\nabla\Phi(t)$ is parallel to the frequency $t$. Every Fourier mode of the mean field then remains aligned with the corresponding gradient mode. Conversely, if $\nabla\Phi(t)$ is not parallel to $t$ at some frequency, a real trigonometric objective can combine such modes to produce a nonsymmetric Jacobian. The Gaussian-product characteristic function depends on the individual singular values of its matrix argument rather than only on its Frobenius norm. Its law is therefore not spherically symmetric at finite rank, and the EGGROLL mean field need not be conservative. The proof and an explicit cosine example appear in \cref{app:conservativity}.

Nonconservativity alone does not necessarily show that the altered field harms optimization. The following construction shows that finite rank can change even the local stability of a strict optimum.

\begin{proposition}[An optimum can become repelling]\label{prop:unstable}
Set $r=\sigma=1$, $\varepsilon=0.01$, and
\begin{equation}
T_1=\begin{pmatrix}-2&-2\\-2&-1\end{pmatrix},\qquad T_2=\begin{pmatrix}-2&2\\2&-1\end{pmatrix}.
\end{equation}
Define
\begin{equation}\label{eq:unstablef}
f(W)=\cos\langle T_1,W\rangle_F+\frac14\cos\langle T_2,W\rangle_F+\varepsilon\sum_{i,j}\cos W_{ij}.
\end{equation}
Then zero is a strict local maximum of both $f$ and $F_{1,1}$, but the Jacobian of $g_{1,1}^f$ at zero has a positive eigenvalue $\lambda_+\approx0.02925$.
\end{proposition}
The construction and its trajectories are detailed in \cref{app:conservativity,fig:witnessdynamics}.
\begin{figure}[t]
\centering
\includegraphics[width=\textwidth]{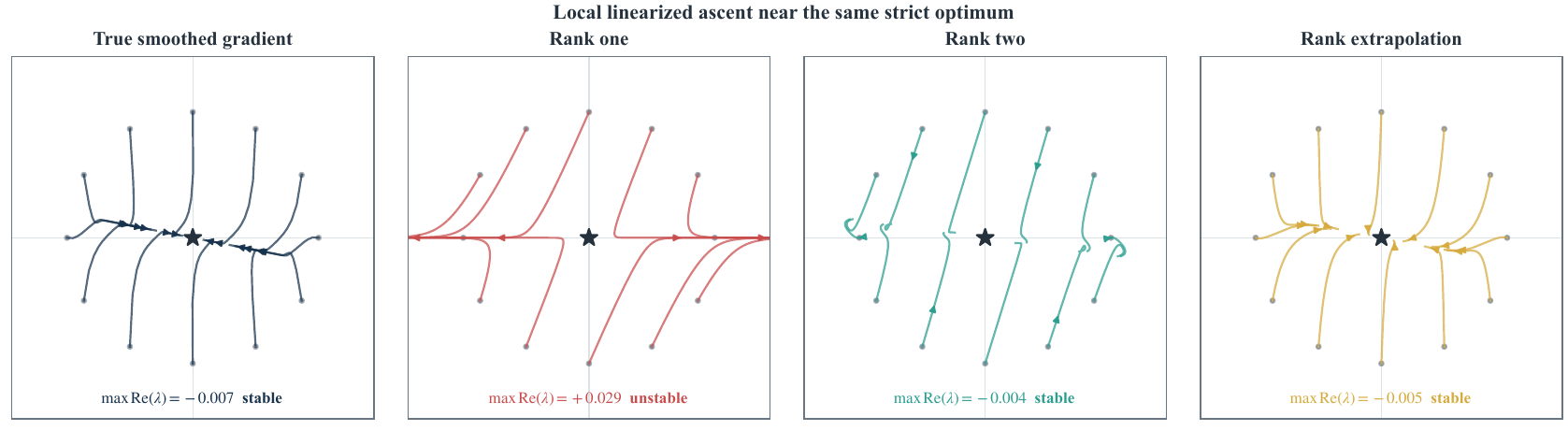}
\caption{Local linearized dynamics for the cosine objective in \cref{eq:unstablef}. Curves solve $\dot\delta=A\delta$, where $A$ is each field's Jacobian at the origin, and are shown in the same plane. The origin is a strict local maximum of both $f$ and $F_{1,1}$. Rank-one EGGROLL repels along one direction, while rank two is locally attracting. The rank-extrapolation comparison is developed in \cref{app:secondaryremedies}.}
\label{fig:witnessdynamics}
\end{figure}

The counterexample establishes that finite rank can change the dynamics, but not which properties of the objective permit this change. Quadratic objectives provide a dividing line. For any quadratic $f$ and any realization of $E_r$, the antithetic difference equals the directional derivative:
\begin{equation}
\frac{f(W+\sigma E_r)-f(W-\sigma E_r)}{2\sigma}=\langle\nabla f(W),E_r\rangle_F.
\end{equation}
Multiplying by $E_r$ and taking expectations recovers $\nabla f(W)$ by identity covariance.
\begin{proposition}[Exactness on quadratic objectives]\label{prop:quadraticexact}
For $f(W)=\tfrac12\langle W,\mathcal HW\rangle_F+\langle C,W\rangle_F+c$ with self-adjoint $\mathcal H$, $g_{r,\sigma}^f(W)=\nabla f(W)$ for every $r\ge1$ and $\sigma>0$.
\end{proposition}
Quadratic exactness clarifies the source of the instability in \cref{prop:unstable}. The objective and its quadratic Taylor approximation have the same value, gradient, and Hessian at the origin, but EGGROLL evaluates the objective at finite perturbations around that point. At $\sigma=1$, these evaluations reach regions where the cosine terms are poorly quadratically approximated. The stability reversal therefore comes from higher-order variation over the perturbation scale. For a general smooth objective, the following result identifies the first such departure from quadratic behavior.
\begin{proposition}[Small-radius expansion]\label{prop:smallradius}
Fix $m,n,r$. If $f\in C^5$ and $M_5=\sup_W\|D^5f(W)\|_{\mathrm{op}}<\infty$, then, uniformly in $W$ as $\sigma\to0$,
\begin{equation}\label{eq:smallradius}
g_{r,\sigma}^f=\nabla f+\frac{\sigma^2}{2}\nabla\Delta f+\frac{\sigma^2}{r}\mathcal Df+R_{r,\sigma},
\quad \sup_W\|R_{r,\sigma}(W)\|_F\le\frac{\sigma^4M_5}{120}\mathbb E[\|E_r\|_F^6],
\end{equation}
where $\Delta f=\sum_{k,\ell}\partial_{k\ell}^2f$ and $(\mathcal Df)_{ij}=\sum_{k,\ell}\partial_{i\ell}\partial_{k\ell}\partial_{kj}f$.
\end{proposition}
The proof is given in \cref{app:local,app:fourthmoments}.
The gradient-Laplacian term is the ordinary cost of dense Gaussian smoothing. The additional third-derivative term is specific to finite rank and need not be a gradient field. Both vanish on quadratics. The expansion concerns the complete displacement $\sigma E_r$, whose root-mean-square Frobenius norm is $\sigma\sqrt{mn}$, so its accuracy depends on dimension and derivative bounds as well as radius. Proofs and nonasymptotic bounds under standard smoothness appear in \cref{app:local}. Conventional convergence consequences are collected in \cref{app:convergence}.

\section{The sampling cost of finite rank}\label{sec:variance}
The preceding sections determine the mean update averaged over the perturbation law. An actual EGGROLL step replaces this expectation with an average over finitely many directions. We now ask how accurately that population estimates its mean. To isolate this sampling variability, we condition on any randomness in fitness evaluation and use the local affine model $f(W+\sigma E)=f(W)+\sigma\langle G,E\rangle_F$, where $G=\nabla f(W)$. One antithetic direction then gives $X=E_r\langle E_r,G\rangle_F$, with $\mathbb E[X]=G$. Because $X$ is unbiased, its total variance, the trace of its covariance, equals $\mathbb E[\|X-G\|_F^2]$.
\begin{theorem}[Finite-population variance]\label{thm:variance}
For independent $X_1,\ldots,X_N$ and $\widehat G_N=N^{-1}\sum_sX_s$,
\begin{equation}\label{eq:populationmse}
\mathbb E[\|\widehat G_N-G\|_F^2]=\frac{\kappa_r}{N}\|G\|_F^2,
\qquad \kappa_r:=mn+1+\frac{2(m+n+1)}r.
\end{equation}
\end{theorem}
The full covariance and proof appear in \cref{app:variance}. Dense Gaussian directions share the factor $mn+1$, while finite rank adds $2(m+n+1)/r$. For a square matrix of width $w$, the relative variance increase over dense Gaussian ES is therefore $2(2w+1)/[r(w^2+1)]=O(1/(rw))$. To give an LLM-scale example, at $w=4096$, as in Qwen3-8B, the rank-one increase is only $0.098\%$. Rank one can therefore estimate its mean field almost as efficiently as dense Gaussian ES estimates its own.

A complementary question is how the mean field itself changes with rank. For $\|u\|_{\infty,p}:=\sup_W\|u(W)\|_F/(1+\|W\|_F)^p$, the kernel expansion in \cref{sec:expansion} gives a uniform statement.
\begin{theorem}[Fixed-radius rank expansion]\label{thm:rankexpansion}
If $|f(W)|\le C_f(1+\|W\|_F^p)$, then for every fixed $\sigma>0$ there exist $R_{\sigma,p}<\infty$ and $r_0$ such that, for $r\ge r_0$,
\begin{equation}\label{eq:rankexpansion}
\left\|g_{r,\sigma}^f-g_{\infty,\sigma}^f-\frac1r\mathcal B_\sigma f\right\|_{\infty,p}\le\frac{R_{\sigma,p}}{r^2}.
\end{equation}
If $\mathcal B_\sigma f$ is nonzero, the $1/r$ rate is tight in this norm.
\end{theorem}
The coefficient $\mathcal B_\sigma f$ identifies the leading finite-rank error, motivating Richardson extrapolation using ranks $r$ and $\alpha r$. However, each estimate retains the dense-Gaussian variance coefficient $mn+1$. Extrapolation therefore improves the bias order while requiring more evaluations and amplifying sampling noise. \Cref{sec:expansion,app:secondaryremedies} gives the proof and develops this tradeoff.

\section{EGGROLL leads to LOO-ROLL}\label{sec:remedies}
The preceding analysis identifies finite-rank bias and the large directional variance already present in dense Gaussian ES. The antithetic implementation adds a third cost: two evaluations per direction. LOO-ROLL reduces this cost by using the rest of the population to construct a leave-one-out baseline. At iteration $t$, sample $N\ge2$ independent directions and write $F_{t,s}=f(W_t+\sigma E_{t,s})$. Define
\begin{equation}\label{eq:looeggroll}
\widehat g_t^{\mathrm{LOO}}=\frac1{N\sigma}\sum_{s=1}^N E_{t,s}(F_{t,s}-\overline F_{t,-s}),
\qquad \overline F_{t,-s}=\frac1{N-1}\sum_{j\ne s}F_{t,j}.
\end{equation}
Since $F_{t,s}-\overline F_{t,-s}=N(F_{t,s}-\overline F_t)/(N-1)$, the baseline requires only the population mean and no additional evaluations.
\begin{proposition}[Unbiasedness of LOO-ROLL]
\label{prop:looroll}
Let $\mathcal F_t$ denote the sigma-field generated by the optimization history before iteration $t$. Conditional on $\mathcal F_t$, suppose that $N\ge2$ and $E_{t,1},\ldots,E_{t,N}$ are independent, centered, and have distribution $\Law(E_r)$. Assume that $F_{t,s}$ and $\norm{E_{t,s}}_F F_{t,s}$ are conditionally integrable. Then,
\begin{equation}
\E\left[\widehat g_t^{\mathrm{LOO}}\mid\mathcal F_t\right]=g_{r,\sigma}^f(W_t).
\end{equation}
\end{proposition}
This unbiasedness statement holds before population score standardization. The practical implementation retains EGGROLL's standardization, whose effect on the expected update is treated in \cref{app:normalization}. To compare variances, we use the same affine model as in \cref{sec:variance}. Write $\widehat g^{\mathrm{LOO}}(W;N)$ and $\widehat g^{\mathrm{anti}}(W;N)$ for the estimators formed from $N$ directions, and let $d=mn$.
\begin{proposition}[Variance of LOO-ROLL in the affine model]\label{prop:loomse}
Suppose $f(W+\sigma E)=f(W)+\sigma\ip{G}{E}_F$ and write $d=mn$, with $\kappa_r$ as defined in \cref{eq:populationmse}. For $N\ge2$ independent directions, the LOO-ROLL estimator satisfies
\begin{equation}
\E\left[\norm{\widehat g^{\mathrm{LOO}}(W;N)-G}_F^2\right]=\left[\frac{\kappa_r}{N}+\frac{d+1}{N(N-1)}\right]\norm{G}_F^2. \label{eq:loomse}
\end{equation}
For $G\ne0$, comparing the variance of $2N$ LOO-ROLL directions with $N$ antithetic directions at equal evaluation cost gives
\begin{equation}
\frac{\E[\norm{\widehat g^{\mathrm{LOO}}(W;2N)-G}_F^2]}{\E[\norm{\widehat g^{\mathrm{anti}}(W;N)-G}_F^2]}=\frac12+\frac{d+1}{2\kappa_r(2N-1)}. \label{eq:looequalmse}
\end{equation}
\end{proposition}
Proofs of both propositions are given in \cref{app:loo}.
Antithetic sampling and baseline subtraction are both established ways to reduce the variance of score-function estimators \citep{salimans2017evolution,williams1992simple,wierstra2014natural}. EGGROLL spends two evaluations on each independent direction and subtracts their fitnesses, canceling reward components shared by the pair. LOO-ROLL instead removes the population reward offset with a leave-one-out baseline, allowing every evaluation to use a new independent direction. At fixed $N$, estimating the baseline adds only an order-$N^{-2}$ term to the order-$N^{-1}$ directional variance. With $2N$ evaluations, LOO-ROLL uses twice as many independent directions and its variance approaches one half of EGGROLL's.

LOO-ROLL preserves EGGROLL's computational structure. Each evaluation uses a low-rank perturbation and the same shared base weights and batched forward computations. The implementation retains seed-based reconstruction, layerwise updates, and weight synchronization. The baseline requires neither additional model-sized state nor communication rounds. Evaluations previously spent on antithetic partners can therefore fund more directions or further updates.

\section{Experiments}\label{sec:experiments}
We begin by testing the theoretical predictions before examining their consequences for LLM post-training and LOO-ROLL. Complete protocols, additional figures, and numerical aggregates appear in \cref{app:empirical,app:experimentdetails}.

\noindent\textbf{Sampling variance and stability follow the predictions.}
Monte Carlo variance agrees to within $0.9\%$ at every tested rank with the exact finite-population prediction $\kappa_r\|G\|_F^2/N$ from \cref{eq:populationmse,tab:variancecheck}. The nonquadratic counterexample reproduces the predicted Jacobian eigenvalue $0.02925$ at rank one and $-0.00444$ at rank two (\cref{fig:witnessdynamics}). Cubic and Fourier objectives recover the $\sigma^2/r$ and fixed-radius $1/r$ corrections. On a quadratic objective, the finite-population formula predicts both the one-step change and the 25-step trajectories from the objective, step size, rank, and population size (\cref{fig:finitepopulationvalidation}). At step size $0.3$ and $N=16$, expected loss grows at rank one but decreases at rank eight and under dense Gaussian ES. Here, the instability arises entirely from sampling variance, since the mean fields agree. \Cref{app:syntheticchecks} gives the full experiments.

\noindent\textbf{Transformer blocks recover the finite-rank surcharge.}
We perturb $16\times16$ attention and MLP sub-blocks of Qwen3-0.6B, evaluating next-token prediction loss on a fixed minibatch. With 8,192 directions and three perturbation seeds, attention-block MSE ratios relative to dense Gaussian ES are $1.264,1.127,1.089,1.028$ at ranks $1,2,4,8$, close to the affine predictions $1.257,1.128,1.064,1.032$ (\cref{fig:fieldaudit}). Across ranks, the cosine similarity between the Monte Carlo estimate of the EGGROLL mean field and the backpropagated block gradient remains between $0.980$ and $0.984$. Thus the experiment resolves the predicted rank-dependent variance while finding little rank dependence in the mean direction at the tested radii. In a separate audit at $N=128$, the mean updates obtained with and without population score standardization have cosine similarity $0.99970$ at rank one (\cref{tab:normalizationaudit}).

\begin{figure*}[!ht]
    \centering
    \includegraphics[width=\textwidth]{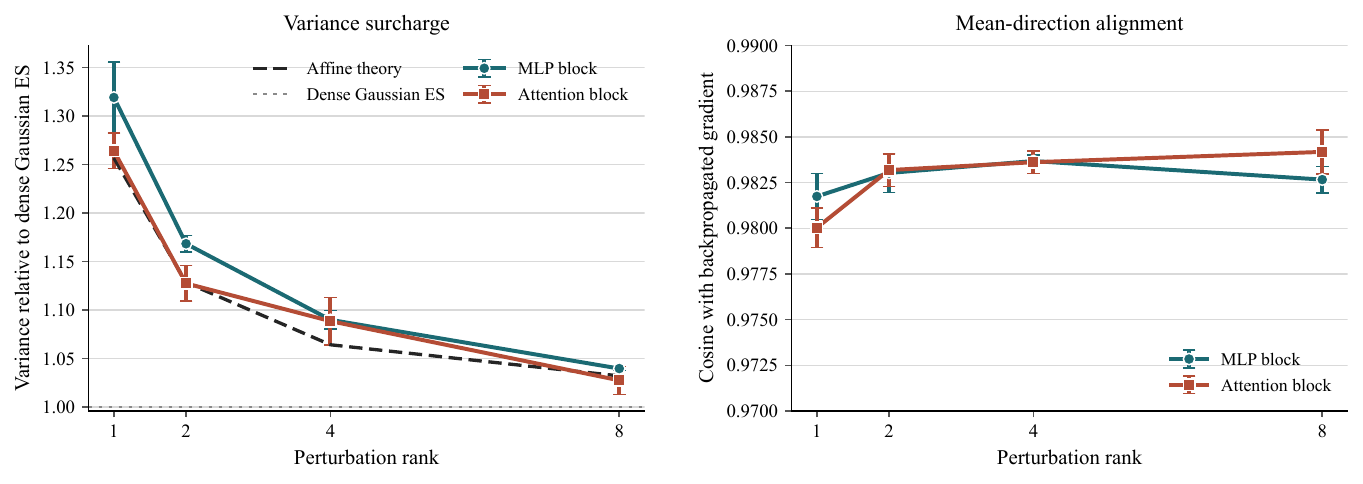}
\caption{Finite-rank audit of two Qwen3-0.6B blocks. \textbf{Left:} single-direction MSE relative to dense Gaussian ES. The dashed curve is the exact affine prediction $[257+66/r]/257$ for a $16\times16$ block. \textbf{Right:} the mean of 8,192 antithetic directions remains aligned with the backpropagated block gradient. Points average four radii within each seed, and bars are standard errors over three seeds.}
    \label{fig:fieldaudit}
\end{figure*}

\noindent\textbf{Increasing rank does not yield a reproducible reward advantage.}
Across Countdown and GSM8K comparisons, rank-eight minus rank-one intervals contain zero, while rank eight increases iteration time by $12$--$25\%$ (\cref{tab:fixedvalidation}). Raising $\sigma$ from $10^{-3}$ to $4\times10^{-3}$ reduces mean reward by $70$--$100\%$ in the 0.6B and 1.7B studies (\cref{fig:endtoend,tab:five-seed}). These results show that performance is substantially more sensitive to radius than rank in these settings.

Having validated the mean-field and variance predictions, we now evaluate LOO-ROLL in complete LLM post-training runs and ask whether its saved evaluations translate into better performance at equal wall time. We compare Qwen3 models from 0.6B to 14B \citep{yang2025qwen3} and SmolLM2-1.7B \citep{allal2025smollm2} over five paired training seeds per setting. Both methods start from the same pretrained checkpoint and retain the same learning rate, rank, radius, prompt batches, and score processing. Equal-iteration runs use the same direction count. Equal-evaluation runs double LOO-ROLL's directions. Equal-time runs continue LOO-ROLL to the paired EGGROLL runtime. Timing includes model loading and checkpoint writing, and accumulates all resumed training.

GSM8K \citep{cobbe2021training} uses all 1,319 test questions and MATH-500 uses all 500 problems. MATH training excludes MATH-500. Countdown and next-token prediction use 256 shared validation examples. Checkpoints are selected by update count or stopping budget, without choosing the best test score. \Cref{tab:looprotocol} provides exact hardware, direction counts, batches, updates, and measured runtimes.

\noindent\textbf{At equal evaluation cost, update MSE is approximately halved.}
Across two transformer blocks, two radii, and three seeds, using the same number of directions increases MSE by at most $5.2\%$, falling to $0.8\%$ at $N=128$. Doubling the directions at equal evaluation cost reduces MSE by $49.0$--$50.1\%$ across $N=16,32,64,128$ (\cref{tab:loorollmse}), agreeing with \cref{eq:looequalmse}.

\begin{table}[t]
\centering
\papertablestyle
\caption{Signed percentage change in LOO-ROLL MSE relative to antithetic EGGROLL. Positive values mean higher MSE. The equal-direction comparison uses $N$ directions for both estimators, while the equal-evaluation comparison uses $2N$ LOO-ROLL directions against $N$ antithetic directions. Values average two transformer blocks, two radii, and three seeds.}
\label{tab:loorollmse}
\begin{tabular*}{\textwidth}{@{\extracolsep{\fill}}rcc@{}}
\toprule
Antithetic directions $N$ & Equal directions & Equal evaluations \\
\midrule
16  & $+5.2\%$ & $\bm{-49.1\%}$ \\
32  & $+2.5\%$ & $\bm{-49.0\%}$ \\
64  & $+1.2\%$ & $\bm{-50.1\%}$ \\
128 & $+0.8\%$ & $\bm{-49.5\%}$ \\
\bottomrule
\end{tabular*}
\end{table}

\begin{table}
\centering
\papertablestyle
\caption{LOO-ROLL minus antithetic EGGROLL reward. GSM8K and MATH-500 results use the full test sets, while Countdown uses 256 shared examples. Values are mean $\pm$ standard deviation of five paired training-seed differences, in percentage points. Bold marks positive mean improvements.}
\label{tab:loorollreward}
\begin{tabular*}{\textwidth}{@{\extracolsep{\fill}}llrrr@{}}
\toprule
Model & Task & Equal iterations & Equal evaluations & Equal wall time \\
\midrule
SmolLM2-1.7B & Countdown & $-0.20\%\pm0.22$ & $-0.08\%\pm0.34$ & $-0.10\%\pm0.46$ \\
SmolLM2-1.7B & GSM8K & $\bm{+0.50\%\pm0.64}$ & $\bm{+0.23\%\pm0.74}$ & $\bm{+4.18\%\pm1.54}$ \\
Qwen3-0.6B & Countdown & $-4.09\%\pm3.97$ & $-1.05\%\pm1.49$ & $-0.62\%\pm3.46$ \\
Qwen3-0.6B & GSM8K & $\bm{+5.35\%\pm1.76}$ & $-4.82\%\pm4.87$ & $\bm{+24.90\%\pm2.38}$ \\
Qwen3-1.7B & Countdown & $-3.89\%\pm1.22$ & $-0.31\%\pm1.56$ & $\bm{+1.68\%\pm1.03}$ \\
Qwen3-1.7B & GSM8K & $0.00\%\pm1.01$ & $-0.26\%\pm0.81$ & $\bm{+2.85\%\pm0.83}$ \\
Qwen3-1.7B & MATH-500 & $\bm{+0.56\%\pm1.17}$ & $-0.08\%\pm1.25$ & $\bm{+12.20\%\pm1.56}$ \\
Qwen3-8B & GSM8K & $\bm{+0.85\%\pm0.66}$ & $-1.08\%\pm1.21$ & $\bm{+14.13\%\pm0.88}$ \\
Qwen3-8B & MATH-500 & $\bm{+0.24\%\pm1.36}$ & $\bm{+0.60\%\pm1.65}$ & $\bm{+8.36\%\pm1.68}$ \\
Qwen3-14B & GSM8K & $\bm{+0.41\%\pm0.89}$ & $-0.08\%\pm0.79$ & $\bm{+7.85\%\pm1.35}$ \\
Qwen3-14B & MATH-500 & $\bm{+1.12\%\pm0.70}$ & $\bm{+0.60\%\pm0.47}$ & $\bm{+8.08\%\pm1.23}$ \\
\bottomrule
\end{tabular*}
\end{table}
\noindent\textbf{Additional updates yield gains across scale and task.}
\Cref{tab:loorollreward,tab:loorolldense} report paired mean changes and sample standard deviations. At matched time, GSM8K accuracy rises from $38.10\%\pm2.03$ to $63.00\%\pm0.80$ at 0.6B, from $65.87\%\pm0.78$ to $80.00\%\pm0.78$ at 8B, and from $71.22\%\pm0.67$ to $79.08\%\pm0.92$ at 14B. On SmolLM2-1.7B, accuracy rises from $18.54\%\pm0.51$ to $22.73\%\pm1.15$. MATH-500 gains are $12.20$, $8.36$, and $8.08$ percentage points at 1.7B, 8B, and 14B. At 14B, accuracy rises from $44.68\%\pm0.72$ to $52.76\%\pm1.13$, with every seed improving. LOO-ROLL completes 522.8 updates on average against EGGROLL's 200 in the same approximately 29-hour budget.

Next-token prediction loss in \cref{tab:loorolldense} also falls at equal time, by $1.20\%$ and $2.16\%$ relative to the EGGROLL baselines at 1.7B and 8B. The corresponding differences at equal evaluations are small, so the improvement follows from the additional updates that the saved computation permits. Across the fourteen primary settings, eleven gains are significant in individual paired tests, nine survive Holm correction, and none significantly favors EGGROLL (\cref{tab:looholm}). The average paired time ratio is approximately $1.001$. At equal evaluations, thirteen differences are unresolved, with a smaller $0.60$-point gain on 14B MATH-500 ($p=0.0459$). The improvements use the EGGROLL learning rate and perturbation radius, without separate tuning for LOO-ROLL. At a fixed direction count, the gap narrows as the population grows. The practical advantage comes from spending the saved evaluations on more directions or updates.

\section{Conclusion}\label{sec:conclusion}\label{sec:discussion}
We showed that finite-rank EGGROLL follows a resolvent-filtered smoothed gradient that can be nonconservative and destabilize an optimum, yet is exact on quadratics and adds little sampling variance for wide matrices. We proposed LOO-ROLL, which replaces antithetic evaluations with leave-one-out baselines while retaining EGGROLL's low-rank, batched, forward-only computation. At equal evaluation cost it approximately halves update MSE, and across fourteen LLM post-training settings with up to 14B parameters, eleven matched-time gains are individually significant and none significantly favors EGGROLL. The main theory is limited to the unstandardized field and a local affine variance model, and the experiments to two model families and three tasks. Future work should treat adaptive score processing and dynamic allocation of rank, radius, and population.

\newpage

\subsection*{AI use statement}
Generative AI tools assisted with literature discovery, experimental orchestration, and language and formatting revisions. The authors independently checked all citations. The authors made all final scientific and editorial decisions and take responsibility for the complete contents of the submission.

\subsection*{Ethics statement}
This work studies optimization methods using publicly available pretrained models and benchmark datasets. It involves no human subjects, private data, or deployment, and raises no ethical concerns beyond those generally associated with language-model research.

\subsection*{Reproducibility statement}
All theorem assumptions are stated in the main text, and complete proofs are provided in the appendix. The supplementary material includes the complete training and evaluation code, experiment configurations, and analysis scripts. LOO-ROLL is implemented as a 37-line patch to the \href{https://github.com/ESHyperscale/eggroll-vllm/tree/bcc215e8784f5f44d24985145c0a71e74283cf1f}{EGGROLL vLLM code at commit \texttt{bcc215e8784f5f44d24985145c0a71e74283cf1f}}, comprising 15 insertions and 22 deletions in one file. The algorithmic patch is provided separately from the shared model-compatibility changes so that the modification can be inspected and reproduced directly. Reward-based runs use rank one, $\sigma=10^{-3}$, and learning rate $2\times10^{-4}$. Next-token prediction uses 32 directions and $\sigma=10^{-2}$ at 0.6B and 1.7B or $3\times10^{-3}$ at 8B, with learning rate $4\times10^{-4}$ only for the 8B setting. Equal-evaluation comparisons double the LOO-ROLL direction count. Experiments use one NVIDIA H100 80GB GPU and 14 CPU cores, except 14B MATH training, which uses two H100 GPUs and 28 CPU cores for both methods. Paired comparisons begin from the same pretrained checkpoint and share seeds, prompt batches, learning rates, ranks, perturbation radii, and score processing. Timing includes model loading, updates, and checkpoint writing, with elapsed time accumulated across resumed runs. Checkpoints are fixed by update count or stopping budget rather than test performance. GSM8K and MATH-500 use their complete 1,319- and 500-example test sets, with MATH-500 excluded from training. Countdown and next-token prediction use 256 shared validation examples. \Cref{tab:looprotocol} reports direction counts, batch sizes, update budgets, and measured runtimes for every primary comparison.


\newpage

\bibliography{iclr2027_conference}
\bibliographystyle{iclr2027_conference}

\newpage

\appendix
\crefalias{section}{appendix}
\crefalias{subsection}{appendix}
\crefalias{subsubsection}{appendix}
\crefname{appendix}{Appendix}{Appendices}
\Crefname{appendix}{Appendix}{Appendices}
\section*{Appendix Table of Contents}
\addcontentsline{toc}{section}{Appendix}
\markboth{Appendix}{Appendix}
\startcontents[appendix]
\printcontents[appendix]{l}{1}{\setcounter{tocdepth}{3}}
\newpage

\section{Additional related work}\label{app:related}
\label{detail:sec:related}

Random-direction finite differences and simultaneous perturbation are classical gradient-free methods \citep{spall1992multivariate,nesterov2017random,duchi2015optimal}. Their nonconvex analyses make the dependence on dimension and smoothing radius explicit \citep{balasubramanian2022zeroth}. MeZO shows that an in-place two-forward-pass implementation can fine-tune large language models with inference-level memory, and subsequent work adds stochastic variance reduction \citep{malladi2023fine,pmlr-v235-gautam24a}. Residual feedback uses a stored function value to obtain a one-query estimator for online zeroth-order optimization \citep{zhang2024boosting}. These methods use dense or coordinatewise perturbations. EGGROLL instead constructs each matrix perturbation as a low-rank Gaussian product and batches the resulting population during inference, giving rise to the finite-rank population field analyzed here.

Baselines have long been used to reduce the variance of score-function and policy-gradient estimators without changing their expectation \citep{williams1992simple}. In REINFORCE leave-one-out, the reward of each sampled response is compared with the average reward of the other responses generated from the same prompt \citep{ahmadian2024back}. LOREN also uses REINFORCE leave-one-out in curvature-aware zeroth-order LLM tuning \citep{seung2026low}. LOO-ROLL specializes the baseline to EGGROLL's Gaussian-product population and evaluates the resulting alternative to its antithetic training estimator.

The original EGGROLL analysis proves a fixed-rank, high-dimensional vanishing-radius guarantee and convergence of its Gaussian-score update to dense Gaussian ES at rate $O(1/r)$ under regularity assumptions \citep{sarkar2026evolution}. Our results concern the complementary regime of finite dimension, finite radius, and finite population. Recent work compares full Gaussian ES with gradient-based LLM post-training and studies its behavior in flat, linear, and quadratic geometries \citep{hoy2026matching,levi2026on}.

Orthogonal and structured perturbations can reduce estimator variance or the cost of applying a direction, while structured control variates use additional problem information \citep{choromanski2018structured,tang2020variance}. Guided and active-subspace ES concentrate exploration in learned subspaces \citep{maheswaranathan2019guided,choromanski2019complexity}. Natural ES and CMA-ES adapt a Gaussian search law \citep{wierstra2014natural,hansen2016cma}. EGGROLL makes a different choice, where the Gaussian-product law is fixed by the desired matrix computation. We analyze the mean field and sampling error that follow from that choice.

Stein operators for scalar products of independent normal variables, and for sums and broader products built from such variables, are well established \citep{gaunt2017stein,gaunt2018products}. These identities act on one scalar product-normal variable. The entries of $AB^\top$ instead share row and column factors, so the matrix law couples different coordinates. Our operator identity keeps this dependence, produces a matrix differential operator acting on the population field, and connects its resolvent to finite-rank optimization dynamics.

\citet{wang2020zeroth} show that convergence can fail for structured zeroth-order perturbations and proposed mixing them with isotropic Gaussian directions. Their use of isotropy concerns an ambient Gaussian component. EGGROLL already has identity covariance, yet its law is not spherically symmetric in the vectorized matrix space. Our conservativity theorem isolates this stronger distributional requirement. Exact spherical symmetry cannot be achieved by a nondegenerate law supported entirely on rank-deficient matrices when $m,n\ge2$, since a uniformly rotated nonzero matrix is full rank almost surely after reshaping. Occasional dense directions could supply an ambient isotropic component, though those evaluations would no longer have the purely low-rank cost of EGGROLL.

\section{Mean-field and finite-rank theory}\label{app:field}

\subsection{Proof of Theorem~\ref{thm:law}}
\begin{proof}[Proof of \cref{thm:law}]
For a pair $(a, b) \in \bR^m \times \bR^n$, condition on $a$:
\begin{equation}
\E_b\big[e^{i\ip{T}{ab^\top}_F/\sqrt r}\big] = \E_b\big[e^{i\tr(T^\top ab^\top)/\sqrt r}\big] = \E_b\big[e^{ib^\top T^\top a/\sqrt r}\big].
\end{equation}
Because we have conditioned on $a$, $c := T^\top a \in \bR^n$ is fixed. Furthermore, since $b \sim \cN(0, I_n)$, $b^\top c$ is a one-dimensional Gaussian, i.e.,
\begin{equation}
b^\top c \sim \cN(0, \norm{c}^2).
\end{equation}
Hence,
\begin{equation}
\E_b\big[e^{ib^\top T^\top a/\sqrt r}\big] = \E_b\big[e^{ib^\top c/\sqrt r}\big] = \exp\Bigl(-\frac{1}{2r}\norm{c}^2\Bigr) = \exp\Bigl(-\frac{1}{2r}\norm{T^\top a}^2\Bigr),
\end{equation}
by the Gaussian characteristic function. We have thus obtained the expectation conditional on $a$. To have the full expectation, we now evaluate
\begin{equation}
\E_a\Bigl[\exp\Bigl(-\frac{1}{2r}\norm{T^\top a}^2\Bigr)\Bigr] = \E_a\Bigl[\exp\Bigl(-\frac{1}{2r}a^\top TT^\top a\Bigr)\Bigr].
\end{equation}
By the Gaussian quadratic-form identity \citep{muirhead2009aspects}, this is
\begin{equation}
\det\Bigl(I_m+\frac{TT^\top}{r}\Bigr)^{-1/2}.
\end{equation}
This means that one rank-one term contributes
\begin{equation}
\E\big[e^{i\ip{T}{ab^\top}_F/\sqrt r}\big] = \det\Bigl(I_m+\frac{TT^\top}{r}\Bigr)^{-1/2}.
\end{equation}
Now,
\begin{equation}
E_r = \frac{1}{\sqrt{r}} \sum_{s=1}^r a_s b_s^\top,
\end{equation}
so
\begin{equation}
\Phi_r(T) = \E\big[e^{i\ip{T}{E_r}_F}\big] = \prod_{s=1}^r \E\big[e^{i\ip{T}{a_sb_s^\top}_F/\sqrt r}\big] = \det\Bigl(I_m+\frac{TT^\top}{r}\Bigr)^{-r/2},
\end{equation}
since all $r$ pairs $(a_s, b_s)$ are independent.
We finish by computing the gradient of the resulting determinant. Write $S=I_m+TT^\top/r$. Since $\log\Phi_r(T)=-(r/2)\log\det S$, its differential is
\begin{equation}
d\log\Phi_r(T)=-\frac{r}{2}\tr(S^{-1}dS), \qquad dS=\frac{(dT)T^\top+T(dT)^\top}{r}.
\end{equation}
Substituting $dS$ cancels the factor $r$ and gives
\begin{equation}
d\log\Phi_r(T)=-\frac12\tr\left(S^{-1}(dT)T^\top+S^{-1}T(dT)^\top\right).
\end{equation}
The matrix $S^{-1}$ is symmetric. Cyclicity of the trace and $\tr(A(dT)^\top)=\tr(A^\top dT)$ therefore make the two terms equal, so
\begin{equation}
d\log\Phi_r(T)=-\tr\left(T^\top S^{-1}dT\right)=\ip{-S^{-1}T}{dT}_F.
\end{equation}
By the defining identity $d h(T)=\ip{\nabla_T h(T)}{dT}_F$ for the Frobenius gradient, $\nabla_T\log\Phi_r(T)=-S^{-1}T$. Finally, $\nabla_T\Phi_r(T)=\Phi_r(T)\nabla_T\log\Phi_r(T)$, which proves \cref{eq:char}.
\end{proof}

\subsection{Proof of Theorem~\ref{thm:resolvent}}
When both fields are square-integrable and $g_{r,\sigma}^f$ belongs to the domain of $\mathcal L$, taking the $L^2$ inner product of \cref{eq:resolvent} with $g_{r,\sigma}^f$ gives the energy identity
\begin{equation}\label{eq:energy}
\left\langle g_{r,\sigma}^f,\nabla F_{r,\sigma}\right\rangle_{L^2}=\|g_{r,\sigma}^f\|_{L^2}^2+\frac{\sigma^2}{r}\|\mathcal L^{1/2}g_{r,\sigma}^f\|_{L^2}^2.
\end{equation}
\begin{proof}[Proof of \cref{thm:resolvent}]
For a tempered distribution $f$, averaging its translations against the perturbation law gives
\begin{equation}
\widehat F_{r,\sigma}(T)=\Phi_r(\sigma T)\widehat f(T), \qquad \widehat g_{r,\sigma}^f(T)=i\Phi_r(\sigma T)\left(I_m+\frac{\sigma^2TT^\top}{r}\right)^{-1}T\widehat f(T)
\end{equation}
in the distributional sense. The Fourier multiplier of $I+(\sigma^2/r)\cL$ is left multiplication by $I_m+\sigma^2TT^\top/r$. Applying it to the second expression gives $iT\widehat F_{r,\sigma}(T)$, the Fourier transform of $\nabla F_{r,\sigma}$. This proves \cref{eq:resolvent} as an identity of tempered distributions.

Under the stated Fourier moment condition, substituting the Fourier representation of $f$ into \cref{eq:field}, interchanging expectation and integration, and using \cref{eq:char} gives the classical formula
\begin{equation}
g_{r,\sigma}^f(W)=\int i\Phi_r(\sigma T)\left(I_m+\frac{\sigma^2TT^\top}{r}\right)^{-1}T e^{i\ip{T}{W}_F}\,\mu_f(dT),
\end{equation}
whereas
\begin{equation}
\nabla F_{r,\sigma}(W)=\int i\Phi_r(\sigma T)Te^{i\ip{T}{W}_F}\,\mu_f(dT).
\end{equation}
The moment condition justifies these operations and makes the identity pointwise. At each frequency $T$, the multiplier $I_m+\sigma^2TT^\top/r$ is symmetric with eigenvalues at least one, making its inverse positive, contractive, and firmly nonexpansive. Parseval's identity equates $L^2$ inner products of fields in parameter space with the corresponding inner products of their Fourier coefficients \citep{folland1999real}. Applying this identity to $(I+(\sigma^2/r)\cL)g=\nabla F$ gives \cref{eq:energy}.
\end{proof}

\subsection{Proof of Corollary~\ref{cor:mismatch}}
\begin{proof}[Proof of \cref{cor:mismatch}]
Set $\alpha=\sigma^2/r$ and write $J=J_{r,\sigma}=(I+\alpha\cL)^{-1}$. On the domain of $\cL$, the resolvent identity gives
\begin{equation}
J-I=-\alpha J\cL.
\end{equation}
To see this directly, take any matrix vector field $h$ for which $\cL h$ is defined. Since $J$ is the inverse of $I+\alpha\cL$,
\begin{equation}
J(h+\alpha\cL h)=h.
\end{equation}
Linearity of $J$ then gives $Jh-h=-\alpha J\cL h$. Applying this equality with $h=\nabla F_{r,\sigma}$ and using $g_{r,\sigma}^f=J\nabla F_{r,\sigma}$ yields
\begin{equation}
g_{r,\sigma}^f-\nabla F_{r,\sigma}=-\frac{\sigma^2}{r}J_{r,\sigma}\cL\nabla F_{r,\sigma},
\end{equation}
which proves the exact formula. The Fourier calculation in \cref{thm:resolvent} also shows that $J$ multiplies every eigendirection at every frequency by a number in $[0,1]$. Parseval's identity therefore gives $\norm{Jh}_{L^2}\le\norm{h}_{L^2}$ for every square-integrable field $h$. Taking $L^2$ norms in the exact formula and applying this inequality to $h=\cL\nabla F_{r,\sigma}$ proves \cref{eq:mismatchbound}.
\end{proof}

\subsection{Proofs and examples for conservativity}\label{app:conservativity}

\subsubsection{Proof of Theorem~\ref{thm:conservative}}
For the proof, write $\Phi(t)=\E[e^{it^\top Z}]$ and define
\begin{equation}
(\cT_{Z,\sigma}f)(x)=\frac1\sigma\E[Zf(x+\sigma Z)].
\end{equation}
We prove the equivalence of the following statements:
\begin{enumerate}[label=(\roman*),leftmargin=1.8em]
\item for every $\sigma>0$ and every real trigonometric polynomial $f$, the field $\cT_{Z,\sigma}f$ is conservative;
\item for every $t\in\R^d$, $\nabla\Phi(t)=\lambda(t)t$ for some complex scalar $\lambda(t)$;
\item $\Phi$ is radial;
\item the law of $Z$ is spherically symmetric.
\end{enumerate}
\begin{proof}[Proof of \cref{thm:conservative}]
We first prove the equivalence of (i) and (ii). Fix $t\in\R^d$ and consider the complex Fourier mode $f_t(x)=e^{it^\top x}$. The real and imaginary parts of this mode are $\cos(t^\top x)$ and $\sin(t^\top x)$, so applying $\cT_{Z,\sigma}$ to $f_t$ records its action on both real modes. From the definition of $\Phi$,
\begin{equation}
\nabla\Phi(\sigma t)=i\E\left[Ze^{i\sigma t^\top Z}\right].
\end{equation}
Consequently,
\begin{equation}
(\cT_{Z,\sigma}f_t)(x)=\frac{e^{it^\top x}}{\sigma}\E\left[Ze^{i\sigma t^\top Z}\right]=\frac{1}{i\sigma}\nabla\Phi(\sigma t)e^{it^\top x}. \label{eq:generalconservativemode}
\end{equation}
A vector field of the form $ce^{it^\top x}$ is a gradient field if and only if its coefficient $c\in\mathbb C^d$ is parallel to $t$. Indeed, if $c=\gamma t$, then
\begin{equation}
ce^{it^\top x}=\nabla\left(\frac{\gamma}{i}e^{it^\top x}\right).
\end{equation}
Conversely, the Jacobian of $ce^{it^\top x}$ is $ict^\top e^{it^\top x}$. Any continuously differentiable gradient field has a symmetric Jacobian because its Jacobian is the Hessian of its potential. For $t\ne0$, symmetry of the present Jacobian is equivalent to $ct^\top=tc^\top$. Choose an index $j$ for which $t_j\ne0$. Comparing the $(i,j)$ entries gives $c_it_j=t_ic_j$ for every $i$, and hence $c=(c_j/t_j)t$. Thus the Jacobian is symmetric only when $c$ and $t$ are parallel. Equation~\eqref{eq:generalconservativemode} is therefore conservative precisely when $\nabla\Phi(\sigma t)$ is parallel to $t$. Since $\sigma>0$ and $t$ are arbitrary, this condition is equivalent to (ii). Linearity then extends the conclusion from individual sine and cosine modes to every trigonometric polynomial, proving (i)$\Leftrightarrow$(ii). The zero-frequency mode causes no exception because $\nabla\Phi(0)=i\E[Z]=0$.

We next prove (ii)$\Leftrightarrow$(iii). Suppose (ii) holds. If $v^\top t=0$, then $v$ is tangent at $t$ to the sphere of radius $\norm{t}$, and
\begin{equation}
v^\top\nabla\Phi(t)=\lambda(t)v^\top t=0.
\end{equation}
Thus $\Phi$ has zero directional derivative along every tangent direction to each sphere. Spheres are connected when $d\ge2$, so $\Phi$ is constant on every sphere and hence depends only on $\norm{t}$. This is precisely radiality. Conversely, if $\Phi$ is radial, then it is constant in every direction tangent to a sphere. The gradient is therefore orthogonal to every tangent direction and must be parallel to the radial direction $t$, which gives (ii).

Finally, (iii)$\Leftrightarrow$(iv) is the standard characteristic-function characterization of spherical symmetry \citep{muirhead2009aspects}. If the law of $Z$ is spherically symmetric, then $QZ$ and $Z$ have the same law for every orthogonal matrix $Q$, and hence $\Phi(Q^\top t)=\Phi(t)$. Therefore $\Phi$ depends only on $\norm{t}$. Conversely, if $\Phi$ is radial, the characteristic function of $QZ$ is
\begin{equation}
\E\left[e^{it^\top QZ}\right]=\Phi(Q^\top t)=\Phi(t).
\end{equation}
Uniqueness of characteristic functions then implies that $QZ$ and $Z$ have the same law, so the law is spherically symmetric. The assumption $d\ge2$ is necessary because in one dimension every continuous scalar field is conservative, regardless of whether the perturbation law is symmetric.
\end{proof}

\subsubsection{An explicit nonconservative mode}
For the Gaussian-product law, spherical symmetry fails because different singular directions are attenuated by different amounts. A single diagonal Fourier mode is enough to display the resulting nonconservativity.
\begin{corollary}[A bounded smooth nonconservative example]
\label{cor:curl}
Let $m=n=2$, $T=\diag(1,2)$, and $f(W)=\cos\ip{T}{W}_F$. Then
\begin{equation}
g_{r,\sigma}^f(W) =-\Phi_r(\sigma T) \left(I+\frac{\sigma^2TT^\top}{r}\right)^{-1}T \sin\ip{T}{W}_F. \label{eq:witness}
\end{equation}
For every finite $r$ and $\sigma>0$, its Jacobian is nonsymmetric at $W=0$ and at every $W$ with $\cos\ip{T}{W}_F\ne0$. Hence, the field is not conservative.
\end{corollary}
\begin{proof}
Recall that $J(T)=(I+\sigma^2TT^\top/r)^{-1}T$. After vectorization, differentiating \cref{eq:witness} gives
\begin{equation}
Dg_{r,\sigma}^f(W)=-\Phi_r(\sigma T)\cos\ip{T}{W}_F\,\operatorname{vec}(J(T))\operatorname{vec}(T)^\top.
\end{equation}
A nonzero rank-one matrix $uv^\top$ is symmetric only if $u$ and $v$ are parallel, since its column and row spaces are spanned by $u$ and $v$, respectively, and these spaces coincide for a symmetric matrix. Here,
\begin{equation}
J(T)=\diag\left(\frac{1}{1+\sigma^2/r},\frac{2}{1+4\sigma^2/r}\right),
\end{equation}
which is not a scalar multiple of $T=\diag(1,2)$ for finite $r$ and $\sigma>0$. The displayed Jacobian is therefore nonsymmetric whenever $\cos\ip{T}{W}_F\ne0$.
\end{proof}

\subsubsection{Proof of Proposition~\ref{prop:unstable}}
\begin{proof}[Proof of \cref{prop:unstable}]
Let $t_q=\operatorname{vec}(T_q)$. The Hessian of $f$ at zero is
\begin{equation}
-t_1t_1^\top-\frac14t_2t_2^\top-\varepsilon I_4\prec0.
\end{equation}
Smoothing multiplies the first two cosine coefficients by $\Phi_1(T_1)=\Phi_1(T_2)=1/(3\sqrt2)$ and every coordinate cosine by $\Phi_1(e_{ij})=1/\sqrt2$:
\begin{equation}
\nabla^2F_{1,1}(0)=-\frac{1}{3\sqrt2}\left(t_1t_1^\top+\frac14t_2t_2^\top\right)-\frac{\varepsilon}{\sqrt2}I_4\prec0.
\end{equation}
Hence, the origin remains a strict local maximum after smoothing.

We next compute the Jacobian of the population field. For a cosine mode with frequency $T$, \cref{eq:witness} gives the contribution
\begin{equation}
-\Phi_1(T)\operatorname{vec}(J(T))\operatorname{vec}(T)^\top
\end{equation}
at the origin. We have
\begin{equation}
J(T_1)=\begin{pmatrix}0&-1/3\\-1/3&1/6\end{pmatrix},
\qquad
J(T_2)=\begin{pmatrix}0&1/3\\1/3&1/6\end{pmatrix}.
\end{equation}
Therefore, the contribution of the first two cosine modes is
\begin{equation}
A_0=-\frac{1}{3\sqrt2}\left(\operatorname{vec}(J(T_1))t_1^\top+\frac14\operatorname{vec}(J(T_2))t_2^\top\right).
\end{equation}
Substituting the displayed matrices and expanding the determinant gives
\begin{equation}
\det(\lambda I_4-A_0)=\lambda^2\left(\lambda^2+\frac{35\sqrt2}{144}\lambda-\frac1{81}\right).
\end{equation}
The positive root of the quadratic factor is
\begin{equation}
\lambda_+(A_0)=\frac{\sqrt{1737}-35}{144\sqrt2}.
\end{equation}
For a coordinate mode $T=e_{ij}$, we have $J(e_{ij})=e_{ij}/2$ and $\Phi_1(e_{ij})=1/\sqrt2$. The corresponding Jacobian contribution is $-\varepsilon\operatorname{vec}(e_{ij})\operatorname{vec}(e_{ij})^\top/(2\sqrt2)$. Summing over the four coordinates gives $-\varepsilon I_4/(2\sqrt2)$. This scalar shift subtracts $\varepsilon/(2\sqrt2)$ from every eigenvalue of $A_0$, so the larger eigenvalue of the full Jacobian is
\begin{equation}
\frac{\sqrt{1737}-35}{144\sqrt2} -\frac{\varepsilon}{2\sqrt2},
\end{equation}
which is positive at $\varepsilon=0.01$.
\end{proof}

\subsection{Proofs and bounds near quadratic objectives}\label{app:local}

\subsubsection{Proof of Proposition~\ref{prop:quadraticexact}}
\begin{proof}[Proof of \cref{prop:quadraticexact}]
Applying the antithetic identity gives
\begin{equation}
g_{r,\sigma}^f(W)=\E\left[E_r\ip{\nabla f(W)}{E_r}_F\right]=\nabla f(W),
\end{equation}
where the last equality follows from $\E[\operatorname{vec}(E_r)\operatorname{vec}(E_r)^\top]=I$.
\end{proof}

\subsubsection{Proof of Proposition~\ref{prop:smallradius}}
\begin{proof}[Proof of \cref{prop:smallradius}]
Taylor expansion of the antithetic update gives
\begin{equation}
g_{r,\sigma}^f(W)=\nabla f(W)+\frac{\sigma^2}{6}\E\left[E_rD^3f(W)[E_r,E_r,E_r]\right]+R_{r,\sigma}(W), \label{eq:antitheticthird}
\end{equation}
where identity covariance was used for the first term. The order-$\sigma^2$ expectation is
\begin{equation}
\E\left[E_rD^3f(W)[E_r,E_r,E_r]\right]=3\nabla\Delta f(W)+\frac{6}{r}\cD f(W). \label{eq:fourthmomentcontraction}
\end{equation}
This identity follows by writing $E_r=r^{-1/2}\sum_{q=1}^r a_qb_q^\top$ and applying the Gaussian fourth-moment formula \citep{isserlis1918formula} separately to the independent row and column factors. The terms that also occur for a dense Gaussian matrix give $3\nabla\Delta f$. The remaining product-law contribution gives $6\cD f/r$. Substitution into \cref{eq:antitheticthird} yields the two order-$\sigma^2$ terms in \cref{eq:smallradius}. Appendix~\ref{app:fourthmoments} provides a more in-depth version of this calculation, including the provenance of the coefficient $1/r$.

Finally, let $M_5=\sup_W\norm{D^5f(W)}_{\mathrm{op}}$. Taylor's theorem bounds the matrix-valued remainder in \cref{eq:antitheticthird} by
\begin{equation}
\norm{R_{r,\sigma}(W)}_F\leq \frac{\sigma^4M_5}{120}\E[\norm{E_r}_F^6].
\end{equation}
\end{proof}

\subsubsection{Nonasymptotic bounds under standard smoothness}\label{sec:stability}
For $d=mn$, the exact fourth moment of the perturbation norm is
\begin{equation}\label{eq:mu4standard}
\mu_4(r):=\mathbb E[\|E_r\|_F^4]=d(d+2)+\frac{2d(m+n+1)}{r}.
\end{equation}
\begin{lemma}[Field error under standard smoothness]\label{lem:standardfield}
Suppose $h$ is $L$-smooth. Then, for every $W$,
\begin{equation}
\norm{g_{r,\sigma}^h(W)-\nabla h(W)}_F
\le \frac{L\sigma}{2}\mu_4(r)^{3/4}. \label{eq:Lsmoothfield}
\end{equation}
If, in addition, the Hessian is $M$-Lipschitz as an operator, meaning that $\norm{\nabla^2h(W)-\nabla^2h(V)}_{\mathrm{op}}\le M\norm{W-V}_F$ for all $W,V$, then
\begin{equation}
\norm{g_{r,\sigma}^h(W)-\nabla h(W)}_F
\le \frac{M\sigma^2}{6}\mu_4(r). \label{eq:Hessianfield}
\end{equation}
\end{lemma}
\begin{proof}
For any matrix $E$, $L$-smoothness gives
\begin{equation}
\left|\frac{h(W+\sigma E)-h(W-\sigma E)}{2\sigma}-\ip{\nabla h(W)}{E}_F\right|\le\frac{L\sigma}{2}\norm{E}_F^2.
\end{equation}
Multiplying by $E=E_r$, taking expectations, and using identity covariance yields \cref{eq:Lsmoothfield}, since $\E[\norm{E_r}_F^3]\le\mu_4(r)^{3/4}$. If the Hessian is $M$-Lipschitz, the second-order Taylor remainders at $W\pm\sigma E$ are bounded by $M\sigma^3\norm{E}_F^3/6$. The quadratic terms cancel in the antithetic difference, leaving
\begin{equation}
\left|\frac{h(W+\sigma E)-h(W-\sigma E)}{2\sigma}-\ip{\nabla h(W)}{E}_F\right|\le\frac{M\sigma^2}{6}\norm{E}_F^3.
\end{equation}
Averaging after multiplication by $E_r$ gives \cref{eq:Hessianfield}.
\end{proof}

\subsection{Complete covariance calculation}\label{app:variance}

The total-variance formula in \cref{thm:variance} follows from the complete entrywise covariance below.

\begin{theorem}[Exact finite-rank covariance]
\label{detail:thm:variance}
For all indices $i,k\in[m]$ and $j,\ell\in[n]$,
\begin{equation}
\Cov(X_{ij},X_{k\ell})
=\norm{G}_F^2\delta_{ik}\delta_{j\ell}+G_{ij}G_{k\ell}+\frac2r\left[
\delta_{ik}(G^\top G)_{j\ell}
+\delta_{j\ell}(GG^\top)_{ik}
+G_{i\ell}G_{kj}
\right].
\label{eq:covariance}
\end{equation}
Define the variance factor
\begin{equation}
\kappa_r:=mn+1+\frac{2(m+n+1)}{r}. \label{eq:variancefactor}
\end{equation}
Then
\begin{equation}
\E[\norm{X-G}_F^2] =\kappa_r\norm{G}_F^2. \label{eq:mse}
\end{equation}
For independent copies $X_1,\ldots,X_N$ and their average $\widehat G_N=\frac1N\sum_{s=1}^NX_s$,
\begin{equation}
\E[\norm{\widehat G_N-G}_F^2]=\frac{1}{N}\E[\norm{X-G}_F^2]=\frac{\kappa_r}{N}\norm{G}_F^2. \label{detail:eq:populationmse}
\end{equation}
\end{theorem}

\begin{proof}
Substitute $E_r=r^{-1/2}\sum_{q=1}^ra_qb_q^\top$ into $\E[X_{ij}X_{k\ell}]$ and apply the Gaussian fourth-moment identity to the independent factors $a_q$ and $b_q$. Summing the resulting fourth-moment terms against the entries of $G$ gives
\begin{equation}
\E[X_{ij}X_{k\ell}]=\norm{G}_F^2\delta_{ik}\delta_{j\ell}+2G_{ij}G_{k\ell}+\frac2r\left[\delta_{ik}(G^\top G)_{j\ell}+\delta_{j\ell}(GG^\top)_{ik}+G_{i\ell}G_{kj}\right].
\end{equation}
Since $\E[X]=G$, subtracting $G_{ij}G_{k\ell}$ proves \cref{eq:covariance}. Taking $i=k$, $j=\ell$, and summing yields
\begin{equation}
\tr\Cov(X)=mn\norm{G}_F^2+\norm{G}_F^2+\frac2r\left[m\norm{G}_F^2+n\norm{G}_F^2+\norm{G}_F^2\right],
\end{equation}
which is \cref{eq:mse}. Independence gives the factor $1/N$ in \cref{eq:populationmse}.
\end{proof}

Dense Gaussian directions and Gaussian-product directions share the variance factor $mn+1$. Finite rank adds only the product-law correction $2(m+n+1)/r$. The relative variance increase of finite-rank EGGROLL over dense Gaussian ES in the affine model is therefore
\begin{equation}
\rho_{m,n,r}=\frac{2(m+n+1)}{r(mn+1)}. \label{detail:eq:relative}
\end{equation}
For a square matrix of width $w$, the rank-one relative increase is $2(2w+1)/(w^2+1)=O(1/w)$. \Cref{tab:qwenvariance} evaluates this quantity at the hidden widths of the Qwen3 models used in this work \citep{yang2025qwen3}. These widths describe, for example, their square attention-output matrices. The $16\times16$ matrices in the transformer audit are extracted sub-blocks chosen to make the finite-rank surcharge measurable, not complete model weight matrices.

\begin{table}[tbp]
\centering
\papertablestyle
\caption{Rank-one variance increase over dense Gaussian ES for square attention-output matrices at Qwen3 hidden widths.}
\label{tab:qwenvariance}
\begin{tabular*}{\textwidth}{@{\extracolsep{\fill}}lrr@{}}
\toprule
Model & Width $w$ & Relative increase \\
\midrule
Qwen3-0.6B & $1024$ & $0.391\%$ \\
Qwen3-1.7B & $2048$ & $0.195\%$ \\
Qwen3-4B   & $2560$ & $0.156\%$ \\
Qwen3-8B   & $4096$ & $0.098\%$ \\
Qwen3-14B  & $5120$ & $0.078\%$ \\
\bottomrule
\end{tabular*}
\end{table}

A conclusion to draw is that a rank-one law can be singular with respect to a dense Gaussian law while having nearly the same affine-model variance. Singularity describes the support of the perturbation law, whereas estimator variance depends on a particular combination of its fourth moments. The rank-dependent contribution grows as $m+n$, while the shared dense-Gaussian contribution grows as $mn$.

For a network, independence confines the additional finite-rank variance to within-block terms. \Cref{app:network} gives the exact formula and shows how each block is weighted by its local gradient energy.

\subsection{Fixed-radius rank expansion}
\label{sec:expansion}

This appendix proves the fixed-radius expansion stated in \cref{thm:rankexpansion}. We hold $\sigma$ fixed and expand the finite-rank mean field around its dense-Gaussian limit as $r\to\infty$, identifying both the leading $1/r$ correction and the $O(r^{-2})$ remainder.

For a fixed frequency $T$, expanding the exact characteristic function as $r\to\infty$ using the matrix-logarithm series \citep{higham2008functions} gives
\begin{equation}
\log\Phi_r(T) =-\frac12\norm{T}_F^2 +\frac{1}{4r}\tr\left((TT^\top)^2\right) -\frac{1}{6r^2}\tr\left((TT^\top)^3\right) +O(r^{-3}). \label{eq:logexpansion}
\end{equation}
The leading term is the logarithm of the characteristic function of a dense Gaussian and depends only on $\norm{T}_F$. The first finite-rank correction also depends on how this norm is distributed across singular directions through $\tr((TT^\top)^2)$. Thus, the same anisotropy that produced the resolvent in \cref{sec:law} appears as the leading $1/r$ error at fixed radius.

The frequency-wise expansion must be strengthened before it can be integrated against a general objective. We therefore expand the score-weighted perturbation law itself and control its remainder in a polynomially weighted integral norm. This yields a uniform field bound for objectives of polynomial growth without requiring a Fourier density for $f$. For a matrix displacement $Z$, define the matrix-valued measure
\begin{equation}
\nu_{r,\sigma}(dZ):=\frac{Z}{\sigma^2}\,\mathbb P(\sigma E_r\in dZ).
\end{equation}
The population field can then be written as
\begin{equation}
g_{r,\sigma}^f(W)=\int f(W+Z)\,\nu_{r,\sigma}(dZ).
\end{equation}
Thus, an expansion of $\nu_{r,\sigma}$ immediately becomes an expansion of the update field. For these kernels, write $\widehat\nu(T)=\int e^{i\ip{T}{Z}_F}\nu(dZ)$, with a positive sign because the field evaluates $f(W+Z)$. The transform of the update measure is then the multiplier derived in \cref{sec:law}:
\begin{equation}\label{eq:multiplier-repeated}
\widehat\nu_{r,\sigma}(T)=i\Phi_r(\sigma T)\left(I_m+\frac{\sigma^2}{r}TT^\top\right)^{-1}T.
\end{equation}
Recall from \cref{sec:law} that $\Phi_r(\sigma T)=\det(I_m+\sigma^2TT^\top/r)^{-r/2}$. We expand both rank-dependent factors in powers of $1/r$:
\begin{equation}\label{eq:expansions}
\begin{gathered}\Phi_r(\sigma T)=e^{-\sigma^2\norm{T}_F^2/2}\left[1+\frac{\sigma^4}{4r}\tr\left((TT^\top)^2\right)+O(r^{-2})\right],\\ \left(I_m+\frac{\sigma^2}{r}TT^\top\right)^{-1}=I_m-\frac{\sigma^2}{r}TT^\top+O(r^{-2}).\end{gathered}
\end{equation}
Consequently,
\begin{equation}
\widehat\nu_{r,\sigma}(T)=i e^{-\sigma^2\norm{T}_F^2/2}T+\frac{1}{r}\widehat\beta_\sigma(T)+O(r^{-2}),
\end{equation}
where the matrix-valued correction kernel $\beta_\sigma$ is defined by
\begin{equation}
\widehat\beta_\sigma(T)=i e^{-\sigma^2\norm{T}_F^2/2}\left[\frac{\sigma^4}{4}\tr\left((TT^\top)^2\right)T-\sigma^2TT^\top T\right],
\label{eq:Bsymbol}
\end{equation}
and its operation on the objective is denoted by
\begin{equation}
\cB_\sigma f(W):=\int f(W+Z)\,\beta_\sigma(dZ).
\end{equation}
The two terms in \cref{eq:Bsymbol} have distinct origins. The term containing $\tr((TT^\top)^2)T$ records the difference between finite-rank and dense-Gaussian smoothing, while the term containing $TT^\top T$ records the score mismatch identified in \cref{sec:law}. The main text states the resulting uniform rank expansion. The constants may depend on $m,n,p,\sigma$, and $C_f$.

\subsection{Proof of Theorem~\ref{thm:rankexpansion}}

\begin{proof}[Proof of \cref{thm:rankexpansion}]
We expand the update kernel in the reciprocal rank $t=1/r$ and bound the remainder before integrating against $f$. First set $\sigma=1$ and write $D=mn$. For $t>0$, define
\begin{equation}
\phi_t(T)=\det(I_m+tTT^\top)^{-1/(2t)}, \qquad A_t(T)=(I_m+tTT^\top)^{-1}, \qquad m_t(T)=i\phi_t(T)A_t(T)T.
\end{equation}
At $t=0$, set $\phi_0(T)=e^{-\norm{T}_F^2/2}$ and $A_0(T)=I_m$. Then $m_{1/r}=\widehat\nu_{r,1}$, $m_0=\widehat\nu_{\infty,1}$, and $\partial_t m_0=\widehat\beta_1$. These functions are smooth at $t=0$ because
\begin{equation}
\ell_t(T):=\log\phi_t(T)=-\frac12\int_0^1\tr\left[TT^\top(I_m+utTT^\top)^{-1}\right]du. \label{eq:loginterpolation}
\end{equation}
In particular, $\partial_t\ell_0=\tr((TT^\top)^2)/4$ and $\partial_t A_0=-TT^\top$, recovering the two terms in \cref{eq:Bsymbol}.

To control the remainder, let $\alpha$ be a multi-index for the $D$ entries of $T$, so that $\partial_T^\alpha$ has total derivative order $|\alpha|$. For every fixed order $k$, there are constants $C_k$ and an integer $K_k$, independent of $0\le t\le1$, such that
\begin{equation}
\norm{\partial_T^\alpha\partial_t^2m_t(T)}_F\le C_k(1+\norm{T}_F)^{K_k}\phi_t(T), \qquad |\alpha|\le k. \label{eq:kernelderivativeenvelope}
\end{equation}
To verify this bound, differentiate \cref{eq:loginterpolation} under the integral. Every derivative of an inverse matrix is obtained from $\partial B^{-1}=-B^{-1}(\partial B)B^{-1}$. All inverse factors here have operator norm at most one, while derivatives of $TT^\top$ are polynomial in $T$. Thus every fixed mixed derivative of $\ell_t$ and $A_t$ has a polynomial bound, uniformly for $0\le t\le1$. Repeated differentiation of $e^{\ell_t}$ leaves the factor $\phi_t$ multiplied by products of these bounded derivatives. Applying the product rule to $m_t=i\phi_t A_tT$ proves \cref{eq:kernelderivativeenvelope}.

The determinant supplies a common decay bound. For $0<t\le t_0\le1$,
\begin{equation}
\phi_t(T)\le(1+t\norm{T}_F^2)^{-1/(2t)}\le(1+t_0\norm{T}_F^2)^{-1/(2t_0)}. \label{eq:kernelcommondecay}
\end{equation}
The first inequality follows by expanding the product of $1+t s_j^2$, where $s_j$ are the singular values of $T$. The second follows because $\log(1+ta)/t$ decreases with $t$ for $a\ge0$. The bound also holds at $t=0$ by continuity. Choose $t_0$ so that $1/t_0>K_k+D/2$. The right-hand side of \cref{eq:kernelderivativeenvelope}, with \cref{eq:kernelcommondecay}, is then square-integrable in $T$, uniformly over $0\le t\le t_0$.

Taylor's formula with an integral remainder therefore gives
\begin{equation}
\partial_T^\alpha(m_t-m_0-t\partial_t m_0)=t^2\int_0^1(1-v)\partial_T^\alpha\partial_t^2m_{vt}\,dv, \qquad \norm{\partial_T^\alpha(m_t-m_0-t\partial_t m_0)}_{L^2}\le C_k' t^2.
\end{equation}
For general fixed $\sigma>0$, the update multiplier is $\sigma^{-1}m_t(\sigma T)$. Rescaling $T$ and taking $t=1/r$ therefore proves, for $r\ge r_0:=\lceil1/t_0\rceil$,
\begin{equation}
\left\|\partial_T^\alpha\left(\widehat\nu_{r,\sigma}-\widehat\nu_{\infty,\sigma}-\frac1r\widehat\beta_\sigma\right)\right\|_{L^2}\le\frac{C_{\sigma,k}}{r^2}, \qquad |\alpha|\le k. \label{eq:kernelsobolev}
\end{equation}

Finally choose an integer $q>D/2$ and take $k=p+q$. Let $\rho_{r,\sigma}$ be the inverse transform of the remainder, using the positive-sign kernel convention of \cref{eq:multiplier-repeated}. Cauchy--Schwarz and Plancherel's identity \citep{folland1999real} imply
\begin{equation}
\int(1+\norm{Z}_F)^p\norm{\rho_{r,\sigma}(Z)}_F\,dZ\le C_{D,p,q}\sum_{|\alpha|\le p+q}\norm{\partial_T^\alpha\widehat\rho_{r,\sigma}}_{L^2}\le\frac{C_{\sigma,p}}{r^2}. \label{eq:weightedkernelremainder}
\end{equation}
Indeed, apply Cauchy--Schwarz with the square-integrable weight $(1+\norm{Z}_F)^{-q}$, then bound the weighted $L^2$ norm by the finitely many monomials $Z^\alpha\rho_{r,\sigma}$ of degree at most $p+q$. Plancherel converts those monomials into Fourier derivatives. The correction kernel $\beta_\sigma$ has a Schwartz density because its transform is a polynomial times a Gaussian. Uniqueness of Fourier transforms of finite measures identifies $\rho_{r,\sigma}(Z)\,dZ$ with $\nu_{r,\sigma}-\nu_{\infty,\sigma}-\beta_\sigma/r$.

Since $|f(W+Z)|\le 2C_f(1+\norm{W}_F)^p(1+\norm{Z}_F)^p$, integration against this remainder gives
\begin{equation}
\norm{g_{r,\sigma}^f(W)-g_{\infty,\sigma}^f(W)-\frac1r\cB_\sigma f(W)}_F\le \frac{2C_f C_{\sigma,p}}{r^2}(1+\norm{W}_F)^p,
\end{equation}
which proves \cref{eq:rankexpansion}. Multiplying the expansion by $r$ shows that $r(g_{r,\sigma}^f-g_{\infty,\sigma}^f)$ converges to $\cB_\sigma f$ in $\norm{\cdot}_{\infty,p}$. If this limit is nonzero, the difference cannot decay faster than $1/r$.
\end{proof}

The expansion identifies the dense-Gaussian limit and the leading finite-rank error $\cB_\sigma f/r$. The two terms in \cref{eq:Bsymbol} respectively arise from the change in smoothing law and the score mismatch, while the remaining error is $O(r^{-2})$.

Since $g_{\infty,\sigma}^f=\nabla F_{\infty,\sigma}$, both fields in this comparison use the same fixed perturbation radius. Their difference therefore isolates the effect of finite rank rather than the smoothing shared with dense Gaussian ES. In particular,
\begin{equation}
\norm{g_{r,\sigma}^f-\nabla F_{\infty,\sigma}}_{\infty,p}\le\frac{K_{\sigma,p}}{r}+\frac{R_{\sigma,p}}{r^2}, \qquad K_{\sigma,p}:=\norm{\cB_\sigma f}_{\infty,p}. \label{eq:rankfieldbound}
\end{equation}
For bounded $f$, this is an ordinary uniform bound. For polynomially growing $f$, the denominator in $\norm{\cdot}_{\infty,p}$ allows the field error to grow at the same polynomial rate with $W$. When $\cB_\sigma f$ is nonzero, the leading term does not vanish, so the $1/r$ rate cannot be improved in this norm. This strengthens the $O(1/r)$ comparison of \citet{sarkar2026evolution}: their Edgeworth argument establishes the rate, whereas the kernel expansion identifies its coefficient and leaves an $O(r^{-2})$ residual. Appendix~\ref{app:convergence} combines the field-error bounds with the finite-population variance from \cref{sec:variance} to derive finite-iteration guarantees for approaching stationary points of the original objective, and, at fixed perturbation radius, of the dense-Gaussian-smoothed objective.

\section{Proofs for LOO-ROLL}\label{app:loo}

\subsection{Proof of Proposition~\ref{prop:looroll}}

\begin{proof}[Proof of \cref{prop:looroll}]
Conditional on $\mathcal F_t$, the term $E_{t,s}F_{t,s}/\sigma$ has expectation $g_{r,\sigma}^f(W_t)$. The baseline $\overline F_{t,-s}$ depends only on the other directions and is therefore conditionally independent of $E_{t,s}$. Centering gives $\mathbb E[E_{t,s}\overline F_{t,-s}\mid\mathcal F_t]=0$. Averaging over $s$ proves the claim.
\end{proof}

\subsection{Proof of Proposition~\ref{prop:loomse}}

\begin{proof}[Proof of \cref{prop:loomse}]
In the affine model, the constant fitness cancels from every centered score. Write $z_i=\operatorname{vec}(E_i)$, $v=\operatorname{vec}(G)$, and $y_i=z_i^\top v$. The vectors $z_i$ are independent, centered, and have identity covariance. With $\bar z=N^{-1}\sum_i z_i$ and $\bar y=N^{-1}\sum_i y_i$, the vectorized estimator is the sample covariance applied to $v$:
\begin{equation}
\operatorname{vec}(\widehat g^{\mathrm{LOO}}(W;N))=\frac{1}{N-1}\sum_{i=1}^N(z_i-\bar z)(y_i-\bar y).
\end{equation}
Expanding the two means gives the following centered decomposition:
\begin{equation}
\operatorname{vec}(\widehat g^{\mathrm{LOO}}(W;N))-v=\frac1N\sum_i(z_i y_i-v)-\frac{1}{N(N-1)}\sum_{i<j}(z_i y_j+z_j y_i).
\end{equation}
Every cross-pair summand has conditional mean zero given either of its two directions. Consequently, these summands are uncorrelated with the first sum and with one another, including when two pairs share an index. By \cref{thm:variance}, $\E[\norm{z_i y_i-v}^2]=\kappa_r\norm{v}^2$. Independence and identity covariance also give
\begin{equation}
\E[\norm{z_i y_j+z_j y_i}^2]=2\E[\norm{z_i}^2]\E[y_j^2]+2\norm{\E[z_i y_i]}^2=2(d+1)\norm{v}^2, \qquad i\ne j.
\end{equation}
There are $N(N-1)/2$ unordered pairs. Adding their variances therefore proves \cref{eq:loomse}. Substituting $2N$ for the LOO-ROLL population and dividing by the antithetic variance $\kappa_r\norm{G}_F^2/N$ proves \cref{eq:looequalmse}.
\end{proof}

\section{Additional LLM comparisons}\label{app:experimentdetails}
The fixed-corpus rank comparisons are reported in \cref{tab:fixedvalidation,tab:denselmranks}, while \cref{tab:loorolldense} reports the next-token prediction results.

\begin{table}
\centering
\papertablestyle
\caption{Fixed-validation changes in held-out task reward from rank one to rank eight, in percentage points, over five matched seeds. Positive values favor rank eight and negative values favor rank one. The evaluation set contains 256 examples and is shared across ranks and budgets. ``Matched evaluations'' uses the same ES iterations and model evaluations. ``Matched wall time'' extends rank-one training to the measured rank-eight runtime. Brackets give paired 95\% Student-$t$ intervals.}
\label{tab:fixedvalidation}
\begin{tabular*}{\textwidth}{@{\extracolsep{\fill}}llrr@{}}
\toprule
Model & Task & Matched evaluations & Matched wall time \\
\midrule
Qwen3-0.6B & Countdown & $-.60\%\;[-3.28,2.07]$ & $-1.57\%\;[-3.29,.15]$ \\
Qwen3-0.6B & GSM8K     & $+2.50\%\;[-.82,5.82]$ & $-5.70\%\;[-14.57,3.16]$ \\
Qwen3-1.7B & Countdown & $+1.50\%\;[-.29,3.30]$ & $+1.88\%\;[-.25,4.01]$ \\
Qwen3-1.7B & GSM8K     & $+.94\%\;[-.76,2.64]$ & $-1.33\%\;[-2.88,.23]$ \\
\bottomrule
\end{tabular*}
\end{table}

\begin{table}
\centering
\papertablestyle
\caption{Percentage change in next-token prediction loss from rank one to rank eight over five paired seeds and 256 fixed held-out sequences. Negative values indicate lower loss at rank eight. Brackets give paired 95\% Student-$t$ intervals. Bold entries exclude zero.}
\label{tab:denselmranks}
\begin{tabular*}{\textwidth}{@{\extracolsep{\fill}}lrr@{}}
\toprule
Model & Matched evaluations & Matched wall time \\
\midrule
Qwen3-0.6B & $\bm{-.58\%\;[-.86,-.30]}$ & $-.11\%\;[-.29,.07]$ \\
Qwen3-1.7B & $\bm{-2.23\%\;[-2.53,-1.92]}$ & $\bm{-.93\%\;[-1.34,-.53]}$ \\
Qwen3-4B   & $\bm{-1.76\%\;[-1.95,-1.57]}$ & $\bm{+1.92\%\;[1.48,2.35]}$ \\
Qwen3-8B   & $\bm{-2.18\%\;[-2.60,-1.77]}$ & $\bm{+2.14\%\;[1.95,2.33]}$ \\
\bottomrule
\end{tabular*}
\end{table}

\begin{table}[tbp]
\centering
\papertablestyle
\caption{Percentage change in fixed-validation negative log likelihood from antithetic EGGROLL to LOO-ROLL. Values are mean $\pm$ sample standard deviation over five paired seeds, computed as $100(\text{LOO-ROLL}-\text{EGGROLL})/\text{EGGROLL}$. Negative values indicate lower loss and favor LOO-ROLL, while positive values favor EGGROLL. Bold marks mean improvements for LOO-ROLL.}
\label{tab:loorolldense}
\begin{tabular*}{\textwidth}{@{\extracolsep{\fill}}lrrr@{}}
\toprule
Model & Equal iterations & Equal evaluations & Equal wall time \\
\midrule
Qwen3-0.6B & $+0.20\%\pm0.44$ & $+0.11\%\pm0.30$ & $\bm{-0.55\%\pm0.46}$ \\
Qwen3-1.7B & $+0.21\%\pm0.58$ & $\bm{-0.10\%\pm0.34}$ & $\bm{-1.20\%\pm0.62}$ \\
Qwen3-8B & $+0.53\%\pm0.57$ & $+0.31\%\pm0.49$ & $\bm{-2.16\%\pm0.81}$ \\
\bottomrule
\end{tabular*}
\end{table}

\section{Network extensions and further theoretical results}\label{app:network}

The main results are stated for one matrix so that the finite-rank mechanism remains visible. A network objective couples many parameter matrices, but EGGROLL samples their perturbations independently. This section gives the corresponding joint resolvent, field-error, and variance statements.

\begin{corollary}[Network resolvent]\label{cor:networkresolvent}
Let $W=(W_1,\ldots,W_Q)$, where $W_q\in\R^{m_q\times n_q}$, and let the blocks $E_{r_q}^{(q)}$ be independent Gaussian-product perturbations. Define
\begin{equation}
\begin{gathered}
F_{\bm r,\sigma}(W)=\E\big[f(W_1+\sigma E_{r_1}^{(1)},\ldots,W_Q+\sigma E_{r_Q}^{(Q)})\big],\\
g_q(W)=\frac1\sigma\E\left[E_{r_q}^{(q)}f(W_1+\sigma E_{r_1}^{(1)},\ldots,W_Q+\sigma E_{r_Q}^{(Q)})\right].
\end{gathered}
\end{equation}
Let $\cL_q$ be the operator in \cref{eq:Ldef}, acting on the $q$th vector-field component and differentiating with respect to $W_q$. If $f$ admits the analogous Fourier representation and moment condition on the product parameter space, then
\begin{equation}
\left(I+\frac{\sigma^2}{r_q}\cL_q\right)g_q=\nabla_{W_q}F_{\bm r,\sigma}\quad\text{for every }q.\label{eq:networkresolvent}
\end{equation}
Equivalently, $g=J_{\bm r,\sigma}\nabla F_{\bm r,\sigma}$ for $J_{\bm r,\sigma}=\operatorname{diag}_q(I+\sigma^2\cL_q/r_q)^{-1}$. The operator $J_{\bm r,\sigma}$ is self-adjoint, positive, contractive, and firmly nonexpansive on the direct-sum $L^2$ space.
\end{corollary}

\begin{proof}
Independence gives the joint characteristic function
\begin{equation}
\Phi_{\bm r}(T_1,\ldots,T_Q)=\prod_{q=1}^Q\Phi_{r_q}(T_q).
\end{equation}
Substitution of the Fourier representation of $f$, followed by differentiation of the $q$th factor and application of \cref{eq:char}, gives the multiplier
\begin{equation}
i\Phi_{\bm r}(\sigma\bm T)\left(I_{m_q}+\frac{\sigma^2T_qT_q^\top}{r_q}\right)^{-1}T_q.
\end{equation}
The multiplier of $\nabla_{W_q}F_{\bm r,\sigma}$ is $iT_q\Phi_{\bm r}(\sigma\bm T)$. The blockwise identity follows, and the remaining properties follow from \cref{thm:resolvent} on each diagonal block.
\end{proof}

For the smoothness bounds, write $d_q=m_qn_q$, $D=\sum_qd_q$, and $\mathbf E=(E_{r_1}^{(1)},\ldots,E_{r_Q}^{(Q)})$. Equip this product space with the direct-sum Frobenius norm $\norm{\mathbf E}^2=\sum_q\norm{E_{r_q}^{(q)}}_F^2$. Independence across blocks gives
\begin{equation}
\mu_4(\bm r):=\E[\norm{\mathbf E}^4]=D(D+2)+\sum_{q=1}^Q\frac{2d_q(m_q+n_q+1)}{r_q}.\label{eq:mu4network}
\end{equation}

\begin{corollary}[Network field error under standard smoothness]\label{cor:networkfield}
Let $h$ be a loss on the product of the block spaces, and define $g_{\bm r,\sigma}^h$ using the joint perturbation $\mathbf E$. If $h$ is $L$-smooth in the direct-sum norm, then
\begin{equation}
\norm{g_{\bm r,\sigma}^h(\mathbf W)-\nabla h(\mathbf W)}\le\frac{L\sigma}{2}\mu_4(\bm r)^{3/4}.\label{eq:networkLsmoothfield}
\end{equation}
If the Hessian is $M$-Lipschitz in the same norm, then
\begin{equation}
\norm{g_{\bm r,\sigma}^h(\mathbf W)-\nabla h(\mathbf W)}\le\frac{M\sigma^2}{6}\mu_4(\bm r).\label{eq:networkHessianfield}
\end{equation}
\end{corollary}

\begin{proof}
The smoothness inequalities in \cref{lem:standardfield} depend only on the norm of the full perturbation. Identity covariance supplies the first-order term, and \cref{eq:mu4network} supplies the required fourth moment.
\end{proof}

\begin{corollary}[Network-level variance]\label{cor:network}
Let block $q=1,\ldots,Q$ have shape $m_q\times n_q$, rank $r_q$, independent perturbation, and gradient $G_q$. Write $D=\sum_qm_qn_q$ and $\norm{G}^2=\sum_q\norm{G_q}_F^2$. The concatenated single-direction estimator satisfies
\begin{equation}
\E[\norm{X-G}^2]=(D+1)\norm{G}^2+\sum_{q=1}^Q\frac{2(m_q+n_q+1)}{r_q}\norm{G_q}_F^2.\label{eq:networkmse}
\end{equation}
\end{corollary}

\begin{proof}
Concatenate the matrix blocks into one vector. The Gaussian part of the fourth moment produces $(D+1)\norm{G}^2$. Independence makes the cross-block fourth cumulants vanish, and \cref{thm:variance} gives the remaining within-block terms.
\end{proof}

The additional variance from block $q$ is weighted by $\norm{G_q}_F^2$. Its contribution therefore depends on both the dimensions of the block and the amount of gradient energy concentrated there. The same formula supplies a network-level relative-noise constant for the convergence results in Appendix~\ref{app:convergence}. One valid choice is
\begin{equation}
C^2=D+1+\max_q\frac{2(m_q+n_q+1)}{r_q}.\label{eq:networkC}
\end{equation}

\subsection{Secondary corrections suggested by the analysis}
\label{app:secondaryremedies}

The fixed-radius expansion and variance calculation suggest two further modifications. Although both are theoretically valid estimators, neither improved on ordinary EGGROLL in our model-scale studies. We record the results in this appendix rather than presenting them as viable practical contributions.

\subsubsection{Rank extrapolation}

The coefficient in \cref{thm:rankexpansion} can be cancelled by combining two ranks. For $s>r$, define
\begin{equation}
g_{r,s,\sigma}^{\mathrm{RE}}:=\frac{s g_{s,\sigma}^f-r g_{r,\sigma}^f}{s-r}. \label{eq:generalrichardson}
\end{equation}
The resulting bias can be bounded explicitly for any chosen rank ratio.

\begin{corollary}[Rank extrapolation at a prescribed tolerance]
\label{cor:rankrichardson}
Under the assumptions of \cref{thm:rankexpansion}, let $s=\alpha r$ for some $\alpha>1$ such that $s$ is an integer. For all sufficiently large $r$,
\begin{equation}
\norm{g_{r,s,\sigma}^{\mathrm{RE}}-g_{\infty,\sigma}^f}_{\infty,p}\le \frac{R_{\sigma,p}}{r^2}\frac{\alpha+1}{\alpha(\alpha-1)}. \label{eq:generalrichardsonbound}
\end{equation}
If $\delta=\varepsilon r^2/R_{\sigma,p}$, it is sufficient to choose
\begin{equation}
\alpha\ge\frac{\delta+1+\sqrt{(\delta+1)^2+4\delta}}{2\delta}. \label{eq:alphamin}
\end{equation}
\end{corollary}

\begin{proof}
Write $g_{k,\sigma}^f=g_{\infty,\sigma}^f+\cB_\sigma f/k+R_k$ with $\norm{R_k}_{\infty,p}\le R_{\sigma,p}/k^2$. Substitution into \cref{eq:generalrichardson} cancels $\cB_\sigma f$ and gives \cref{eq:generalrichardsonbound}. Solving $(\alpha+1)/(\alpha(\alpha-1))\le\delta$ gives \cref{eq:alphamin}.
\end{proof}

Cancellation of the mean-field term comes at the cost of estimating two fields, and the extrapolation weights amplify their sampling errors. If the two estimates are independent and use $N_r$ and $N_s$ directions with single-direction variances $V_r$ and $V_s$, then
\begin{equation}
\E\left[\norm{\widehat g_{r,s,\sigma}^{\mathrm{RE}}-g_{r,s,\sigma}^{\mathrm{RE}}}_F^2\right]=\frac{s^2V_s/N_s+r^2V_r/N_r}{(s-r)^2}. \label{eq:richardsonvariance}
\end{equation}
At a fixed budget of 256 fitness evaluations, transformer-block experiments compare ordinary rank-one EGGROLL with extrapolation using ranks $(1,2)$ and $(1,4)$. Across two blocks and three radii, the extrapolated estimator has respectively $7.63$--$8.67$ and $2.35$--$2.62$ times the MSE of ordinary EGGROLL. The leading bias is below the sampling resolution in these settings, so cancelling it does not repay the variance introduced by dividing the population between two ranks.

\subsubsection{Variance reduction with a reference point}

The dominant term in \cref{eq:mse} is shared with dense Gaussian ES. Following zeroth-order variance reduction for language-model fine-tuning \citep{pmlr-v235-gautam24a}, let $\xi$ denote the sampled data and define
\begin{equation}
Y_{r,\sigma}(W;\xi,E):=E\frac{\ell(W+\sigma E;\xi)-\ell(W-\sigma E;\xi)}{2\sigma}, \qquad g_{r,\sigma}(W):=\E\left[Y_{r,\sigma}(W;\xi,E)\right].
\end{equation}
At a reference point $\widetilde W$, let $\overline g_B(\widetilde W)$ be an independent unbiased average of $B$ updates and set
\begin{equation}
Y^{\mathrm{VR}}(W,\widetilde W):=Y_{r,\sigma}(W;\xi,E)-Y_{r,\sigma}(\widetilde W;\xi,E)+\overline g_B(\widetilde W). \label{eq:eggrollsvrg}
\end{equation}
The corresponding control-variate guarantee for Gaussian-product directions is as follows.

\begin{proposition}[Low-rank variance reduction]
\label{prop:eggrollsvrg}
The estimator in \cref{eq:eggrollsvrg} is unbiased for $g_{r,\sigma}(W)$. If
\begin{equation}
\E\left[\norm{Y_{r,\sigma}(W;\xi,E)-Y_{r,\sigma}(V;\xi,E)}_F^2\right]\le L_Y^2\norm{W-V}_F^2
\end{equation}
and $\E[\norm{\overline g_B(\widetilde W)-g_{r,\sigma}(\widetilde W)}_F^2]\le V_{\mathrm{ref}}/B$, then
\begin{equation}
\E\left[\norm{Y^{\mathrm{VR}}(W,\widetilde W)-g_{r,\sigma}(W)}_F^2\right]\le L_Y^2\norm{W-\widetilde W}_F^2+\frac{V_{\mathrm{ref}}}{B}. \label{eq:eggrollsvrgbound}
\end{equation}
\end{proposition}

\begin{proof}
Taking expectations in \cref{eq:eggrollsvrg} cancels the field at the reference point and leaves $g_{r,\sigma}(W)$. Independence of the reference estimate makes the centered errors orthogonal, and the two assumed bounds give \cref{eq:eggrollsvrgbound}.
\end{proof}

We tested this estimator on Qwen3-0.6B and Qwen3-1.7B with five paired seeds using next-token prediction loss and GSM8K rewards. \Cref{tab:svrgpilot} reports held-out reward relative to matched EGGROLL checkpoints. All four means favor ordinary EGGROLL, with the 1.7B next-token prediction loss difference excluding zero. These results do not establish a practical benefit from maintaining the reference estimate and its additional evaluations.

\begin{table}[tbp]
\centering
\papertablestyle
\caption{SVRG improvement over EGGROLL over five paired seeds. For next-token prediction, entries are the reduction in held-out loss, computed as EGGROLL loss minus SVRG loss. For GSM8K, entries are SVRG reward minus EGGROLL reward in percentage points. Positive values favor SVRG and negative values favor EGGROLL. Brackets give paired 95\% Student-$t$ intervals, and bold entries have intervals that exclude zero.}
\label{tab:svrgpilot}
\begin{tabular*}{\textwidth}{@{\extracolsep{\fill}}llr@{}}
\toprule
Model & Objective & Difference \\
\midrule
Qwen3-0.6B & Next-token prediction loss & $-.0161\;[-.0488,.0166]$ \\
Qwen3-0.6B & GSM8K & $-6.17\%\;[-17.43,5.08]$ \\
Qwen3-1.7B & Next-token prediction loss & $\bm{-.0197\;[-.0321,-.0072]}$ \\
Qwen3-1.7B & GSM8K & $-1.25\%\;[-2.98,.48]$ \\
\bottomrule
\end{tabular*}
\end{table}

\subsection{Convergence consequences of the field bounds}\label{app:convergence}

This appendix translates the mean-field error and finite-population variance bounds into standard stationarity guarantees under smoothness.

\begin{lemma}[Descent with accumulated field error]\label{thm:meanstationarity}
Suppose $h$ is $L$-smooth and bounded below by $h_\star$. Let $g=\nabla h+b$ and $x_{t+1}=x_t-\eta g(x_t)$. If $0<\eta\le1/L$, then
\begin{equation}
\frac1T\sum_{t=0}^{T-1}\norm{\nabla h(x_t)}^2\le\frac{2(h(x_0)-h_\star)}{\eta T}+\frac1T\sum_{t=0}^{T-1}\norm{b(x_t)}^2.\label{eq:meanstationarity}
\end{equation}
\end{lemma}

\begin{proof}
Write $a_t=\nabla h(x_t)$ and $b_t=b(x_t)$. Smoothness and the identity $-\ip{a_t}{a_t+b_t}=-\norm{a_t}^2/2-\norm{a_t+b_t}^2/2+\norm{b_t}^2/2$ give
\begin{equation}
h(x_{t+1})\le h(x_t)-\frac\eta2\norm{a_t}^2-\frac\eta2(1-L\eta)\norm{a_t+b_t}^2+\frac\eta2\norm{b_t}^2.
\end{equation}
Summation proves the claim.
\end{proof}
Define
\begin{equation}
B_{r,\sigma}^{(L)}:=\frac{L\sigma}{2}\mu_4(r)^{3/4},\qquad B_{r,\sigma}^{(M)}:=\frac{M\sigma^2}{6}\mu_4(r).\label{eq:standardbiases}
\end{equation}
Combining \cref{thm:meanstationarity} with \cref{lem:standardfield} gives the following direct consequence.

\begin{theorem}[Mean-field convergence under standard smoothness]\label{thm:explicitmeanstationarity}
Suppose $h$ is $L$-smooth and bounded below by $h_\star$. Fix $\varepsilon>0$, use $\eta=1/L$, and set $H_0=h(x_0)-h_\star$. Either of the sufficient choices
\begin{equation}
0<\sigma\le\sigma_L(r,\varepsilon):=\frac{\sqrt2\,\varepsilon}{L\mu_4(r)^{3/4}},\label{eq:sigmastandardL}
\end{equation}
or, when the Hessian is $M$-Lipschitz,
\begin{equation}
0<\sigma\le\sigma_M(r,\varepsilon):=\left(\frac{3\sqrt2\,\varepsilon}{M\mu_4(r)}\right)^{1/2}\label{eq:sigmastandardM}
\end{equation}
ensures $T^{-1}\sum_{t=0}^{T-1}\norm{\nabla h(x_t)}^2\le\varepsilon^2$ for $T\ge\lceil4LH_0/\varepsilon^2\rceil$.
\end{theorem}

\begin{proof}
The applicable radius makes the corresponding field-error bound at most $\varepsilon/\sqrt2$. Substitution into \cref{eq:meanstationarity} completes the proof.
\end{proof}

\subsubsection{Finite-population convergence}
\label{sec:dynamics}

A finite-population update contains sampling noise around the mean field and systematic error relative to a chosen reference gradient. The next result separates these contributions, allowing either the original objective or dense Gaussian ES at the same radius to serve as the reference. As in \cref{sec:stability}, we minimize a loss $h$ through $x_{t+1}=x_t-\eta\widehat g_t$.

\begin{theorem}[Nonconvex convergence under relative population noise]
\label{thm:nonconvex}
Suppose $h$ is $L$-smooth and bounded below by $h_\star$. Let $(\mathcal F_t)_{t\ge0}$ be the filtration generated by $x_0$ and all perturbation and evaluation randomness revealed before iteration $t$, so that $x_t$ is $\mathcal F_t$-measurable. At every iterate, assume
\begin{align}
\E[\widehat g_t\mid\mathcal F_t]&=g(x_t),
\label{eq:relativeoraclemean}\\
\E[\norm{\widehat g_t-g(x_t)}^2\mid\mathcal F_t]
&\le \frac{G^2+C^2\norm{g(x_t)}^2}{N}
\label{eq:relativeoraclevariance}
\end{align}
for $C,G\ge0$. If $0<\eta\le[L(1+C^2/N)]^{-1}$, then
\begin{equation}
\frac1T\sum_{t=0}^{T-1}\E[\norm{\nabla h(x_t)}^2] \le \frac{2(h(x_0)-h_\star)}{\eta T} +\frac1T\sum_{t=0}^{T-1}\E[\norm{g(x_t)-\nabla h(x_t)}^2]+\frac{L\eta G^2}{N}. \label{eq:nonconvex}
\end{equation}
\end{theorem}

\begin{proof}
Smoothness and conditional expectation give
\begin{equation}
\E_t[h(x_{t+1})] \le h(x_t)-\eta\ip{\nabla h(x_t)}{g(x_t)} +\frac{L\eta^2}{2}\left[\left(1+\frac{C^2}{N}\right)\norm{g(x_t)}^2+\frac{G^2}{N}\right].
\end{equation}
Write $g=\nabla h+b$ and use
\begin{equation}
-\ip{\nabla h}{g}=-\frac12\norm{\nabla h}^2-\frac12\norm{g}^2+\frac12\norm{b}^2.
\end{equation}
The step-size condition makes the total coefficient of $\norm{g}^2$ nonpositive. Taking full expectation and summing over $t$ gives \cref{eq:nonconvex}.
\end{proof}

We first measure stationarity with respect to the original loss $h$. In this comparison, the field error includes both smoothing and finite-rank score mismatch, and the standard smoothness bounds from \cref{sec:benign} control them together.

\begin{corollary}[Convergence relative to the original loss]\label{cor:explicitnonconvex}
Suppose that $h$ is $L$-smooth and bounded below by $h_\star$, that the finite-population estimator satisfies \cref{eq:relativeoraclemean,eq:relativeoraclevariance} with $g=g_{r,\sigma}^h$, and set $0<\eta\le[L(1+C^2/N)]^{-1}$. Define $B_{r,\sigma}^{\mathrm{std}}=B_{r,\sigma}^{(L)}$. If the Hessian is $M$-Lipschitz, one may instead take $B_{r,\sigma}^{\mathrm{std}}=B_{r,\sigma}^{(M)}$. Then
\begin{equation}
\frac1T\sum_{t=0}^{T-1}\E[\norm{\nabla h(x_t)}^2]
\le\frac{2(h(x_0)-h_\star)}{\eta T}+\left(B_{r,\sigma}^{\mathrm{std}}\right)^2+\frac{L\eta G^2}{N}. \label{eq:explicitnonconvex}
\end{equation}
If $\sigma\le\sigma_L(r,\varepsilon)$, or $\sigma\le\sigma_M(r,\varepsilon)$ under the Lipschitz-Hessian assumption, then the middle term is at most $\varepsilon^2/2$.
\end{corollary}

\begin{proof}
Apply \cref{lem:standardfield} to the field-error term in \cref{eq:nonconvex}.
\end{proof}

Dividing the target tolerance among mean-field error, the finite optimization horizon, and irreducible population noise gives an explicit iteration complexity for a single matrix. Replacing $\mu_4(r)$ by $\mu_4(\bm r)$ from \cref{eq:mu4network} gives the corresponding network statement.

\begin{corollary}[Explicit finite-population complexity]
\label{cor:explicitcomplexity}
Under the assumptions of \cref{cor:explicitnonconvex}, fix $\varepsilon>0$ and let $\Delta=h(x_0)-h_\star$. Set $\eta=1/(2L)$ and choose
\begin{equation}
0<\sigma\le\frac{2\varepsilon}{\sqrt3\,L\mu_4(r)^{3/4}}. \label{eq:complexitysigmaL}
\end{equation}
If the Hessian is $M$-Lipschitz, the alternative choice
\begin{equation}
0<\sigma\le\left(\frac{2\sqrt3\,\varepsilon}{M\mu_4(r)}\right)^{1/2} \label{eq:complexitysigmaM}
\end{equation}
is sufficient. If
\begin{equation}
N\ge\max\left\{C^2,\frac{3G^2}{2\varepsilon^2}\right\} \label{eq:complexityN}
\end{equation}
and
\begin{equation}
T\ge\left\lceil\frac{12L\Delta}{\varepsilon^2}\right\rceil, \label{eq:complexityT}
\end{equation}
then
\begin{equation}
\frac1T\sum_{t=0}^{T-1}\E\left[\norm{\nabla h(x_t)}^2\right]\le\varepsilon^2.
\end{equation}
\end{corollary}

\begin{proof}
The radius choices make the field-error bound at most $\varepsilon/\sqrt3$. The requirement $N\ge C^2$ gives $\eta\le[L(1+C^2/N)]^{-1}$, while the second part of \cref{eq:complexityN} bounds the irreducible term by $G^2/(2N)\le\varepsilon^2/3$. Finally, \cref{eq:complexityT} bounds the optimization term by $4L\Delta/T\le\varepsilon^2/3$. Substitution into \cref{eq:explicitnonconvex} proves the claim.
\end{proof}

The guarantees above concern the original loss. To compare finite-rank EGGROLL directly with dense Gaussian ES at the same radius, let $H_{\infty,\sigma}(x):=\E[h(x+\sigma Z)]$ for $Z\sim\cN(0,I)$. Using this smoothed loss as the reference removes the smoothing error shared by both methods and leaves the finite-rank discrepancy derived in \cref{sec:expansion}.

\begin{corollary}[Convergence relative to dense Gaussian ES]\label{cor:denseesconvergence}
Suppose $h$ satisfies the assumptions of \cref{thm:rankexpansion}, $H_{\infty,\sigma}$ is $L_\sigma$-smooth and bounded below by $H_{\infty,\sigma}^\star$, and the finite-population estimator satisfies \cref{eq:relativeoraclemean,eq:relativeoraclevariance} with $g=g_{r,\sigma}^h$. For the constants $K_{\sigma,p},R_{\sigma,p},r_0$ in \cref{eq:rankfieldbound}, every $r\ge r_0$ and $0<\eta\le[L_\sigma(1+C^2/N)]^{-1}$ satisfy
\begin{equation}
\begin{aligned}\frac1T\sum_{t=0}^{T-1}\E\left[\norm{\nabla H_{\infty,\sigma}(x_t)}^2\right]
&\le\frac{2(H_{\infty,\sigma}(x_0)-H_{\infty,\sigma}^\star)}{\eta T}+\frac{L_\sigma\eta G^2}{N}\\
&\quad+\left(\frac{K_{\sigma,p}}{r}+\frac{R_{\sigma,p}}{r^2}\right)^2\frac1T\sum_{t=0}^{T-1}\E\left[(1+\norm{x_t})^{2p}\right].\end{aligned}\label{eq:denseesconvergence}
\end{equation}
If the displayed parameter moments are uniformly bounded along the trajectory, the mean-field contribution of finite rank to the squared-stationarity guarantee is $O(r^{-2})$ whenever the leading coefficient is nonzero. For bounded fitness, $p=0$ and no trajectory-moment condition is needed.
\end{corollary}

\begin{proof}
Equation~\eqref{eq:rankfieldbound} gives
\begin{equation}
\norm{g_{r,\sigma}^h(x_t)-\nabla H_{\infty,\sigma}(x_t)}\le\left(\frac{K_{\sigma,p}}r+\frac{R_{\sigma,p}}{r^2}\right)(1+\norm{x_t})^p.
\end{equation}
Square this bound and apply \cref{thm:nonconvex} with $H_{\infty,\sigma}$ as the reference loss.
\end{proof}

The two reference objectives should not be conflated. Relative to $h$, the error contains smoothing and finite-rank effects and must be controlled through both $\sigma$ and $r$. Relative to dense Gaussian ES at the same $\sigma$, the shared smoothing disappears and the mean-field difference is $O(1/r)$. The finite-population noise is the same in either comparison.

The constants in the variance assumption can be connected to quantities that are calculable or measurable for EGGROLL. The constant $G^2$ allows noise to remain when the mean update vanishes. The term $C^2\norm{g(x)}^2$ is relative or multiplicative noise, also called a strong-growth component \citep{vaswani2019fast,wang2023zeroth}. For one matrix in the exact affine model, \cref{thm:variance} gives $C^2=mn+1+2(m+n+1)/r$. Appendix~\ref{app:network} gives the corresponding network expression. Relative noise restricts the stable step size, but it vanishes with the mean field and creates no additional error floor. The additive component $G^2$ produces the final term in \cref{eq:nonconvex}. Setting $C=0$ recovers the usual bounded-variance assumption, while the exact affine EGGROLL model has $G=0$.

\subsection{Fourth moments of the Gaussian-product law}
\label{app:fourthmoments}

This appendix supplies the fourth-moment calculation used in \cref{prop:smallradius}. The coefficient in \cref{eq:fourthmomentcontraction} follows from the fourth moment of four entries of $E_r$. At rank one, $(E_1)_{ij}=a_i b_j$, where $a$ and $b$ are independent standard Gaussian vectors. For arbitrary row indices $i_1,\ldots,i_4$ and column indices $j_1,\ldots,j_4$,
\begin{equation}
\E\left[\prod_{s=1}^4(E_1)_{i_sj_s}\right]=\left(\delta_{i_1i_2}\delta_{i_3i_4}+\delta_{i_1i_3}\delta_{i_2i_4}+\delta_{i_1i_4}\delta_{i_2i_3}\right)
\left(\delta_{j_1j_2}\delta_{j_3j_4}+\delta_{j_1j_3}\delta_{j_2j_4}+\delta_{j_1j_4}\delta_{j_2j_3}\right). \label{eq:rankonefourth}
\end{equation}
Each parenthesis contains the three possible pairings of four Gaussian factors. Expanding their product gives nine terms. In three of them, the row and column indices use the same pairing. For example, pairing the first entry with the second and the third with the fourth in both factors gives
\begin{equation}
\delta_{i_1i_2}\delta_{j_1j_2}\delta_{i_3i_4}\delta_{j_3j_4}.
\end{equation}
These three terms are precisely the fourth moment of a dense Gaussian matrix with independent entries. The remaining six terms use different pairings for the row and column indices, such as
\begin{equation}
\delta_{i_1i_2}\delta_{i_3i_4}\delta_{j_1j_3}\delta_{j_2j_4},
\end{equation}
and arise from the rank-one factorization.

For general rank, write $E_r=r^{-1/2}\sum_{q=1}^r a_qb_q^\top$. A nonzero fourth-moment term either draws all four factors from one summand or draws two factors from each of two summands. The three dense-Gaussian terms receive both kinds of contributions, with total coefficient
\begin{equation}
\frac{r}{r^2}+\frac{r(r-1)}{r^2}=1.
\end{equation}
Each of the six additional terms requires all four factors to come from the same rank-one summand and therefore has coefficient $r/r^2=1/r$. Consequently, the fourth moment of $E_r$ consists of the three dense-Gaussian terms plus the six additional terms scaled by $1/r$. To apply this identity in \cref{eq:antitheticthird}, write the $(i_0,j_0)$ entry as
\begin{equation}
\sum_{i_1,j_1,i_2,j_2,i_3,j_3}\partial_{i_1j_1}\partial_{i_2j_2}\partial_{i_3j_3}f(W)\,
\E\left[(E_r)_{i_0j_0}(E_r)_{i_1j_1}(E_r)_{i_2j_2}(E_r)_{i_3j_3}\right].
\end{equation}
The three dense-Gaussian terms give the same quantity, $\partial_{i_0j_0}\Delta f(W)$, after relabeling indices. The six additional terms give the same quantity, $(\cD f(W))_{i_0j_0}$, by symmetry of the third derivative. Summing them proves
\begin{equation}
\E\left[E_rD^3f(W)[E_r,E_r,E_r]\right]=3\nabla\Delta f(W)+\frac6r\cD f(W).
\end{equation}

\section{Complete empirical results}\label{app:empirical}
\subsection{Training budgets}

\Cref{tab:looprotocol} gives the direction counts, batch sizes, update budgets, runtimes, and achieved time matching for all fourteen primary comparisons.

\begin{table}[tbp]
\centering
\papertablestyle
\caption{Training budgets for the fourteen LOO-ROLL comparisons. $N$ counts independent EGGROLL directions, $B$ is the training batch size, $T_E$ is the EGGROLL update count, and $T_L$ is the mean number of LOO-ROLL updates completed at equal time. Hours are mean EGGROLL runtimes over five seeds. NTP denotes next-token prediction.}
\label{tab:looprotocol}
\begin{tabular*}{\textwidth}{@{\extracolsep{\fill}}llrrrrrr@{}}
\toprule
Model & Task & $N$ & $B$ & $T_E$ & $T_L$ & Hours & Time ratio \\
\midrule
Qwen3-0.6B & Countdown & 128 & 8 & 150 & 374.4 & 2.668 & 0.9970 \\
Qwen3-0.6B & GSM8K & 128 & 4 & 150 & 368.4 & 4.145 & 0.9978 \\
Qwen3-0.6B & NTP & 32 & 4 & 80 & 131.8 & 0.151 & 1.0088 \\
Qwen3-1.7B & Countdown & 128 & 8 & 200 & 496.0 & 4.187 & 0.9980 \\
Qwen3-1.7B & GSM8K & 128 & 4 & 200 & 485.6 & 6.190 & 0.9985 \\
Qwen3-1.7B & NTP & 32 & 4 & 160 & 268.0 & 0.293 & 1.0199 \\
Qwen3-1.7B & MATH & 128 & 4 & 200 & 482.8 & 8.838 & 0.9985 \\
Qwen3-8B & GSM8K & 128 & 4 & 120 & 303.6 & 8.929 & 1.0009 \\
Qwen3-8B & NTP & 32 & 4 & 120 & 203.4 & 0.363 & 1.0271 \\
Qwen3-8B & MATH & 128 & 4 & 200 & 484.2 & 18.686 & 0.9998 \\
Qwen3-14B & GSM8K & 128 & 4 & 120 & 323.4 & 14.544 & 0.9675 \\
Qwen3-14B & MATH & 128 & 4 & 200 & 522.8 & 28.948 & 0.9998 \\
SmolLM2-1.7B & Countdown & 128 & 8 & 200 & 485.0 & 4.521 & 0.9981 \\
SmolLM2-1.7B & GSM8K & 128 & 4 & 200 & 444.8 & 3.250 & 0.9981 \\
\bottomrule
\end{tabular*}
\end{table}

\subsection{Paired statistical comparisons}\label{app:loostatistics}

The main text reports mean performance under matched evaluation and wall-time budgets. Here, we test the paired seed-level differences underlying those comparisons. \Cref{tab:looevaltests} gives the individual two-sided tests at equal evaluation cost. \Cref{tab:looholm} gives both the individual and Holm-adjusted $p$-values for the fourteen equal-wall-time comparisons, treating them as one family.

\begin{table}[tbp]
\centering
\papertablestyle
\caption{Two-sided paired tests at equal evaluation cost, with five paired seeds and full-test GSM8K and MATH-500 evaluation. The 14B MATH-500 comparison favors LOO-ROLL with $p=0.04591$. The other comparisons have $p\ge0.05$.}
\label{tab:looevaltests}
\begin{tabular*}{\textwidth}{@{\extracolsep{\fill}}llr@{}}
\toprule
Model & Task & Individual $p$ \\
\midrule
Qwen3-0.6B & Countdown & 0.19198 \\
Qwen3-0.6B & GSM8K & 0.09112 \\
Qwen3-0.6B & Next-token prediction & 0.46333 \\
Qwen3-1.7B & Countdown & 0.68251 \\
Qwen3-1.7B & GSM8K & 0.51541 \\
Qwen3-1.7B & MATH-500 & 0.89277 \\
Qwen3-1.7B & Next-token prediction & 0.53072 \\
Qwen3-8B & GSM8K & 0.11829 \\
Qwen3-8B & MATH-500 & 0.46159 \\
Qwen3-8B & Next-token prediction & 0.23088 \\
Qwen3-14B & GSM8K & 0.84124 \\
Qwen3-14B & MATH-500 & 0.04591 \\
SmolLM2-1.7B & Countdown & 0.64704 \\
SmolLM2-1.7B & GSM8K & 0.53016 \\
\bottomrule
\end{tabular*}
\end{table}

\begin{table}[tbp]
\centering
\papertablestyle
\caption{Two-sided paired tests for the fourteen equal-time LOO-ROLL comparisons. Holm adjustment treats all fourteen settings as one family. Bold adjusted values are below $0.05$, and all nine corresponding effects favor LOO-ROLL.}
\label{tab:looholm}
\begin{tabular*}{\textwidth}{@{\extracolsep{\fill}}llrr@{}}
\toprule
Model & Task & Individual $p$ & Holm-adjusted $p$ \\
\midrule
Qwen3-0.6B & GSM8K & $0.000020$ & $\bm{0.00026}$ \\
Qwen3-1.7B & GSM8K & $0.001532$ & $\bm{0.0123}$ \\
Qwen3-1.7B & MATH-500 & $0.000062$ & $\bm{0.00074}$ \\
Qwen3-8B & GSM8K & $0.000004$ & $\bm{0.00006}$ \\
Qwen3-8B & MATH-500 & $0.000373$ & $\bm{0.00335}$ \\
Qwen3-14B & GSM8K & $0.000199$ & $\bm{0.00199}$ \\
Qwen3-14B & MATH-500 & $0.000125$ & $\bm{0.00137}$ \\
SmolLM2-1.7B & Countdown & $0.654584$ & $1.000000$ \\
SmolLM2-1.7B & GSM8K & $0.003708$ & $\bm{0.0260}$ \\
Qwen3-0.6B & Countdown & $0.710583$ & $1.000000$ \\
Qwen3-1.7B & Countdown & $0.022238$ & $0.088953$ \\
Qwen3-0.6B & Next-token prediction loss & $0.055755$ & $0.167266$ \\
Qwen3-1.7B & Next-token prediction loss & $0.012190$ & $0.060951$ \\
Qwen3-8B & Next-token prediction loss & $0.003985$ & $\bm{0.0260}$ \\
\bottomrule
\end{tabular*}
\end{table}

\subsection{Rank comparisons and transformer audits}

\Cref{tab:five-seed,tab:primaryfull,tab:pairedfull,tab:confirmationfull} report the complete rank and radius comparisons summarized in \cref{fig:endtoend}, including the stress-radius runs omitted from the main fixed-validation comparison. Rank columns give mean final training fitness $\pm$ standard error over three independent seeds, with higher values indicating better performance. The experiments use commit \texttt{bcc215e} of the \href{https://github.com/ESHyperscale/eggroll-vllm}{eggroll-vLLM} repository with the LOO-ROLL modifications described in \cref{sec:remedies}.

\begin{figure}[tbp]
    \centering
    \includegraphics[width=\textwidth]{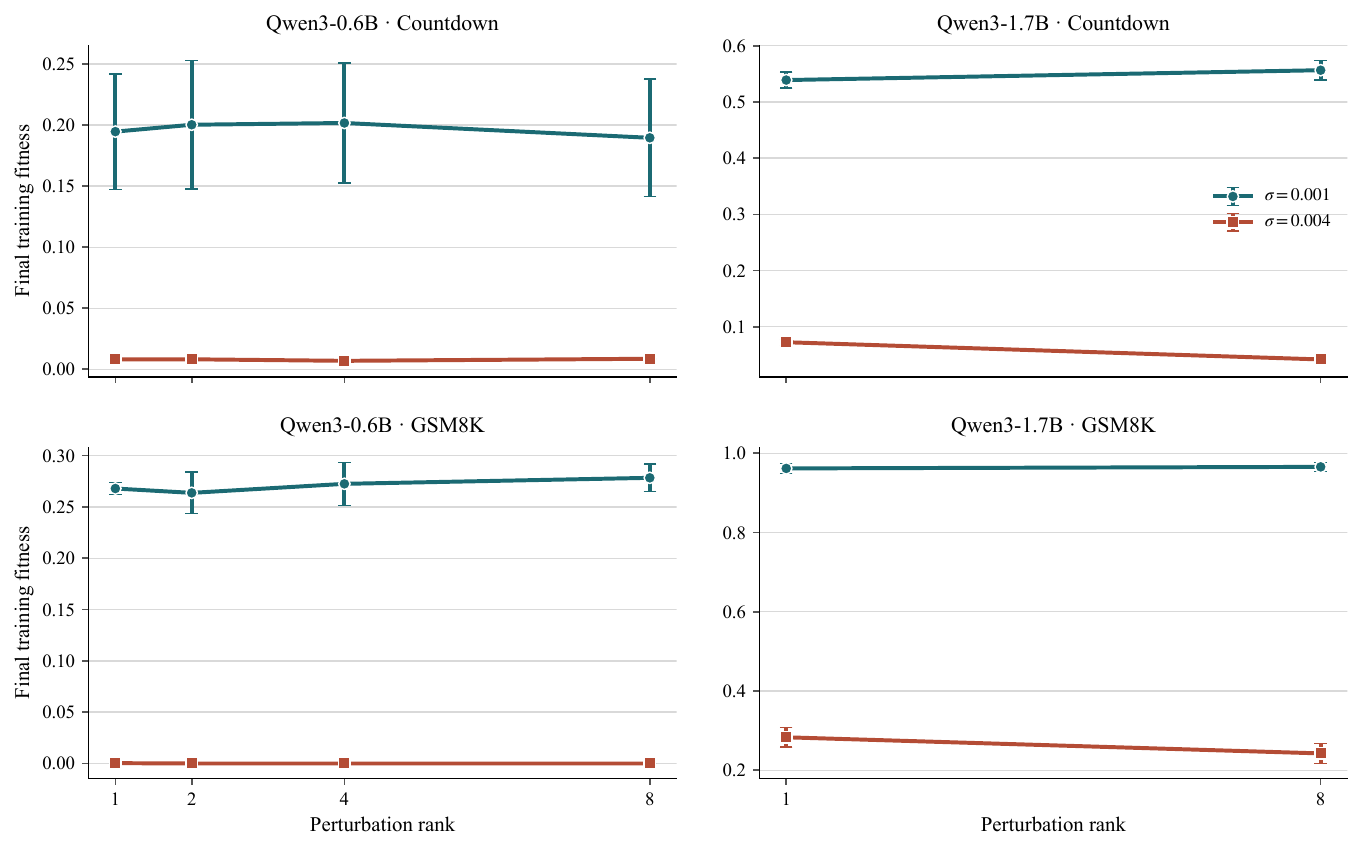}
\caption{Mean final training fitness for the 0.6B and 1.7B rank studies. Bars are standard errors over three seeds. At $\sigma=10^{-3}$, rank differences are small relative to seed variation. Increasing the radius to $\sigma=4\times10^{-3}$ produces a much larger decline in fitness across ranks.}
    \label{fig:endtoend}
\end{figure}

\begin{table}[tbp]
\centering
\papertablestyle
\caption{Expanded useful-radius comparison over five matched seeds. Reward changes are rank eight minus rank one in percentage points, so positive values favor rank eight. Wall overhead is the percentage increase in mean iteration time from rank one to rank eight. Brackets give paired 95\% Student-$t$ intervals. The bold entry excludes zero.}
\label{tab:five-seed}
\begin{tabular*}{\textwidth}{@{\extracolsep{\fill}}llrr@{}}
\toprule
Model & Task & Reward change & Rank-8 wall overhead \\
\midrule
Qwen3-0.6B & Countdown & $-.69\%\;[-2.73,1.34]$ & $+12\%$ \\
Qwen3-0.6B & GSM8K     & $+2.09\%\;[-.23,4.41]$ & $+12\%$ \\
Qwen3-1.7B & Countdown & $+.28\%\;[-1.61,2.17]$ & $+22\%$ \\
Qwen3-1.7B & GSM8K     & $\bm{+.76\%\;[.07,1.46]}$ & $+12\%$ \\
Qwen3-4B   & GSM8K     & $+.21\%\;[-1.30,1.73]$ & $+25\%$ \\
Qwen3-8B   & GSM8K     & $-.70\%\;[-1.58,.18]$ & $+19\%$ \\
\bottomrule
\end{tabular*}
\end{table}

\begin{table}[tbp]
\centering
\papertablestyle
\caption{Complete Qwen3-0.6B rank study. Entries are mean final training fitness $\pm$ standard error over three seeds. Higher values are better.}
\label{tab:primaryfull}
\begin{tabular*}{\textwidth}{@{\extracolsep{\fill}}llrrrr@{}}
\toprule
Task & $\sigma$ & Rank 1 & Rank 2 & Rank 4 & Rank 8 \\
\midrule
Countdown & $.001$ & $.195\pm.047$ & $.200\pm.053$ & $.202\pm.049$ & $.190\pm.048$ \\
Countdown & $.004$ & $.0081\pm.0007$ & $.0082\pm.0007$ & $.0069\pm.0007$ & $.0086\pm.0004$ \\
GSM8K & $.001$ & $.268\pm.006$ & $.264\pm.020$ & $.272\pm.021$ & $.278\pm.013$ \\
GSM8K & $.004$ & $.0003\pm.0003$ & $.0000\pm.0000$ & $.0000\pm.0000$ & $.0000\pm.0000$ \\
\bottomrule
\end{tabular*}
\end{table}

\begin{table}[tbp]
\centering
\papertablestyle
\caption{Paired rank-eight minus rank-one reward changes over three matched seeds, in percentage points. Positive values favor rank eight. Brackets give two-sided 95\% Student-$t$ intervals, and bold entries exclude zero.}
\label{tab:pairedfull}
\begin{tabular*}{\textwidth}{@{\extracolsep{\fill}}llrr@{}}
\toprule
Model & Task & $\sigma=.001$ & $\sigma=.004$ \\
\midrule
Qwen3-0.6B & Countdown & $-.51\%\;[-5.83,4.82]$ & $+.04\%\;[-.42,.51]$ \\
Qwen3-0.6B & GSM8K     & $+1.04\%\;[-5.42,7.50]$ & $-.03\%\;[-.17,.11]$ \\
Qwen3-1.7B & Countdown & $+1.75\%\;[-.15,3.66]$ & $\bm{-3.06\%\;[-5.18,-.94]}$ \\
Qwen3-1.7B & GSM8K     & $+.39\%\;[-1.50,2.29]$ & $\bm{-4.07\%\;[-4.44,-3.70]}$ \\
\bottomrule
\end{tabular*}
\end{table}

\begin{table}[tbp]
\centering
\papertablestyle
\caption{Complete Qwen3-1.7B rank study. Rank columns give mean final training fitness $\pm$ standard error over three seeds. Higher fitness is better. The final column gives the paired rank-eight minus rank-one reward change in percentage points.}
\label{tab:confirmationfull}
\begin{tabular*}{\textwidth}{@{\extracolsep{\fill}}llrrr@{}}
\toprule
Task & $\sigma$ & Rank 1 & Rank 8 & Reward change \\
\midrule
Countdown & $.001$ & $.539\pm.014$ & $.556\pm.017$ & $+1.75\%$ \\
Countdown & $.004$ & $.073\pm.005$ & $.042\pm.004$ & $-3.06\%$ \\
GSM8K & $.001$ & $.962\pm.013$ & $.966\pm.011$ & $+.39\%$ \\
GSM8K & $.004$ & $.283\pm.025$ & $.243\pm.025$ & $-4.07\%$ \\
\bottomrule
\end{tabular*}
\end{table}

The training results combine mean-field effects, population noise, and the resulting optimization trajectory. The transformer-block audit isolates the first two at a fixed model state. \Cref{tab:fieldfull} compares each finite-rank estimator with dense Gaussian ES through its single-direction MSE and compares its estimated mean with the backpropagated gradient through cosine similarity. The rank-one MSE is $26$--$32\%$ above the dense baseline and falls to $3$--$4\%$ above it at rank eight, while the estimated mean retains a cosine near $0.98$ at every rank. Finite-rank variance is therefore visible in these small blocks, whereas the mean direction changes little over the tested radii.

\begin{table}[tbp]
\centering
\papertablestyle
\caption{Comparison of finite-rank and dense-Gaussian ES updates in two Qwen3-0.6B transformer blocks. Relative MSE is the single-direction squared error about the backpropagated block gradient, divided by the corresponding dense-Gaussian MSE, so $1$ denotes equal sampling error and lower values are better. Mean-field cosine is the cosine similarity between the average of 8,192 antithetic updates and the backpropagated block gradient, so $1$ denotes perfect directional agreement. Values average four perturbation radii within each of three independent perturbation seeds and are reported as means $\pm$ standard errors over seeds.}
\label{tab:fieldfull}
\begin{tabular*}{\textwidth}{@{\extracolsep{\fill}}llrrrr@{}}
\toprule
Quantity & Block & Rank 1 & Rank 2 & Rank 4 & Rank 8 \\
\midrule
Relative MSE & MLP & $1.319\pm.036$ & $1.168\pm.009$ & $1.090\pm.009$ & $1.040\pm.002$ \\
Relative MSE & Attention & $1.264\pm.018$ & $1.127\pm.018$ & $1.089\pm.025$ & $1.028\pm.015$ \\
Mean-field cosine & MLP & $.9817\pm.0013$ & $.9830\pm.0011$ & $.9837\pm.0003$ & $.9827\pm.0007$ \\
Mean-field cosine & Attention & $.9800\pm.0011$ & $.9832\pm.0009$ & $.9836\pm.0006$ & $.9842\pm.0012$ \\
\bottomrule
\end{tabular*}
\end{table}

\subsection{Controlled numerical experiments}\label{app:syntheticchecks}

These experiments examine the parts of the theory that have exact numerical answers. Computing the covariance directly from sampled perturbation entries reproduces the tensor predicted by \cref{eq:covariance}, without evaluating the formula itself. Independent calculations also reproduce the resolvent multiplier at individual frequencies, the two Hessians in \cref{prop:unstable}, the cubic identity in \cref{prop:smallradius}, and the $1/r$ rank expansion in \cref{thm:rankexpansion}. Together, these calculations verify the algebra before any model or optimization effects enter.

The finite-sampling experiment asks whether Monte Carlo estimates reproduce those exact quantities. \Cref{tab:variancecheck} uses $100{,}000$ directions for a $3\times5$ affine problem. The empirical single-direction variance is within $0.9\%$ of \cref{eq:mse} at every rank. The same calculation recovers the stability transition in \cref{fig:witnessdynamics}: the largest Jacobian real part is $0.02925$ at rank one and $-0.00444$ at rank two.

\begin{table}[tbp]
\centering
\papertablestyle
\caption{Monte Carlo verification of the exact single-direction variance in \cref{eq:mse}. Results use $100{,}000$ independent sampled directions for a normalized $3\times5$ affine gradient.}
\label{tab:variancecheck}
\begin{tabular*}{\textwidth}{@{\extracolsep{\fill}}rrrr@{}}
\toprule
Rank $r$ & Exact variance & Monte Carlo variance & Relative error \\
\midrule
1 & $34.00$ & $34.24$ & $0.70\%$ \\
2 & $25.00$ & $25.21$ & $0.86\%$ \\
4 & $20.50$ & $20.62$ & $0.57\%$ \\
8 & $18.25$ & $18.19$ & $0.35\%$ \\
\bottomrule
\end{tabular*}
\end{table}

The quadratic case provides an exact test of the finite-population dynamics. Let $h(W)=\norm{W}_F^2/2$, so that the exact population field is $W$ by \cref{prop:quadraticexact}. If $\widehat G_N$ averages $N$ directions, unbiasedness and \cref{eq:populationmse} give
\begin{equation}
\E\left[\norm{W-\eta\widehat G_N}_F^2\mid W\right]=\left[1-2\eta+\eta^2\left(1+\frac{\kappa_r}{N}\right)\right]\norm{W}_F^2.
\end{equation}
Drawing a fresh independent population at every step and iterating this identity yields
\begin{equation}
\frac{\E[\norm{W_T}_F^2]}{\norm{W_0}_F^2}=\left[1-2\eta+\eta^2\left(1+\frac{\kappa_r}{N}\right)\right]^T. \label{eq:quadraticfinitepopulation}
\end{equation}
We test an $8\times8$ problem at $\eta=0.3$ using 4,000 independent one-step trials per rank--population pair and 5,000 independent 25-step trajectories.

\begin{figure*}[tbp]
    \centering
    \includegraphics[width=\textwidth]{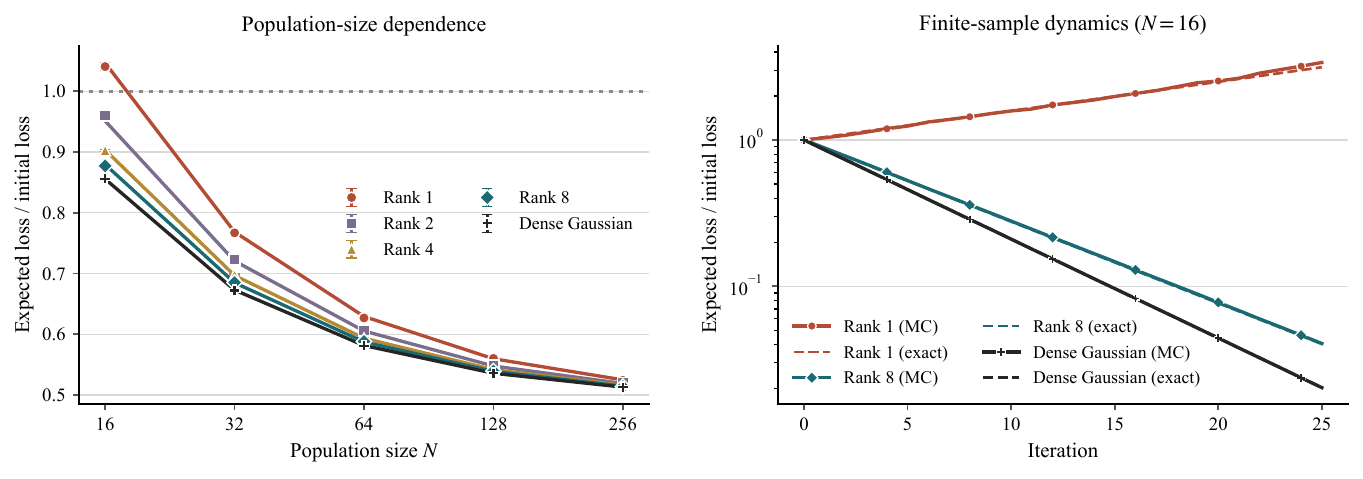}
\caption{Finite-sample validation of \cref{eq:quadraticfinitepopulation}. \textbf{Left:} Monte Carlo one-step ratios agree with the exact curves across populations and ranks. The horizontal line marks stability. \textbf{Right:} at $N=16$, the formula predicts the full mean-square trajectory. Rank one is unstable, whereas rank eight and dense Gaussian ES contract. Marked curves are Monte Carlo means and dashed curves are exact.}
    \label{fig:finitepopulationvalidation}
\end{figure*}

The synthetic result makes the finite-population prediction visible in a controlled setting. With few directions, the term $\kappa_r/N$ can move the iteration across the stability boundary. Increasing $N$ removes this difference, while increasing the matrix width makes the rank-one contribution small relative to the common $mn+1$ term. The transformer experiments in \cref{sec:experiments} examine these predictions in a nonlinear model.

\section{Population centering and standardization}
\label{app:normalization}

The operator analysis concerns the uncentered and unstandardized antithetic estimator, while the end-to-end experiments retain the population centering and standardization specified by EGGROLL. We therefore distinguish the direction of one realized update from the expectation of the normalized estimator. The proposition shows that prompt-wise centering and one common scale do not rotate a realized update. The corollary nevertheless shows that a scale computed from the same random population can change the expected field.

\begin{proposition}[Population centering and common standardization preserve each realized direction]
\label{prop:normalization}
Let pair $j$ use perturbations $+E_j$ and $-E_j$, with prompt-level fitnesses $Y_{j,+,p}$ and $Y_{j,-,p}$. For arbitrary prompt offsets $c_p$, define
\begin{equation}
Z_{j,\pm}=\frac1P\sum_{p=1}^P(Y_{j,\pm,p}-c_p).
\end{equation}
Then $Z_{j,+}-Z_{j,-}=P^{-1}\sum_p(Y_{j,+,p}-Y_{j,-,p})$ exactly. If all $Z_{j,\pm}$ are subsequently divided by the same random scalar $S>0$, the ES update is exactly $S^{-1}$ times the unstandardized update. Consequently, per-prompt population centering has no effect on the update, and standardization changes only its realized magnitude.
\end{proposition}

\begin{proof}
The offsets cancel algebraically in every pair difference. The ES update is linear in those differences, so dividing every score by the same positive $S$ multiplies their weighted sum by $S^{-1}$ and cannot change its realized direction.
\end{proof}

\begin{corollary}[Bias induced by a same-population normalizer]
\label{cor:normalizerbias}
At a fixed iterate, let $\widehat g=g+\xi$ with deterministic $g$ and $\E[\xi]=0$, and let $\alpha=S^{-1}$ have finite second moment. Then
\begin{equation}
\left\|\E[\alpha\widehat g]-\E[\alpha]g\right\| \le \sqrt{\Var(\alpha)\E[\norm{\xi}^2]}. \label{eq:normalizerbias}
\end{equation}
If $\alpha$ is measurable with respect to a sigma-field $\mathcal G$ and $\E[\xi\mid\mathcal G]=0$, the left side is zero. Independence of $\alpha$ and $\xi$ is sufficient. A same-population normalizer can otherwise change the expected field.
\end{corollary}

\begin{proof}
Since $\E[\xi]=0$,
\begin{equation}
\E[\alpha\widehat g]-\E[\alpha]g=\E[(\alpha-\E[\alpha])\xi].
\end{equation}
Cauchy--Schwarz gives \cref{eq:normalizerbias}. If $\alpha$ is $\mathcal G$-measurable and $\E[\xi\mid\mathcal G]=0$, then $\E[\alpha\xi]=\E[\alpha\E[\xi\mid\mathcal G]]=0$.
\end{proof}

The bound becomes explicit in the population size once the update mean and empirical scale satisfy concentration bounds.

\begin{proposition}[$N$-dependent standardization bias]
\label{prop:normalizerN}
For each population size $N$, write $\widehat g_N=g+\xi_N$ and $\alpha_N=S_N^{-1}$. Suppose that, for constants $V,A<\infty$ independent of $N$,
\begin{equation}
\E[\norm{\xi_N}^2]\le\frac{V}{N}, \qquad \Var(\alpha_N)\le\frac{A}{N}.
\label{eq:normalizerconcentration}
\end{equation}
Then
\begin{equation}
\left\|\E[\alpha_N\widehat g_N]-\E[\alpha_N]g\right\|\le\frac{\sqrt{AV}}{N}.
\label{eq:normalizerN}
\end{equation}
Consequently, same-population standardization changes the raw population field by an $O(N^{-1})$ term, apart from the scalar factor $\E[\alpha_N]$.
\end{proposition}

\begin{proof}
Apply \cref{eq:normalizerbias} and substitute the two bounds in \cref{eq:normalizerconcentration}.
\end{proof}

The first condition in \cref{eq:normalizerconcentration} is the ordinary $N^{-1}$ variance reduction from averaging independent member updates. The second follows, for example, when the empirical scale concentrates at rate $N^{-1/2}$ around a positive population scale and its inverse has uniformly controlled second moments. A deterministic floor on the empirical scale supplies a direct sufficient condition. More precisely, if $S_N\ge s_0>0$ almost surely, the population scale is $s>0$, and $\E[(S_N-s)^2]\le C_S/N$, then
\begin{equation}
\Var(S_N^{-1})\le \E[(S_N^{-1}-s^{-1})^2]\le\frac{C_S}{s_0^2s^2N},
\end{equation}
so \cref{eq:normalizerN} holds with $A=C_S/(s_0^2s^2)$. Thus, the raw resolvent field remains the leading expected direction as the population grows.

\subsection{Transformer-block audit}

The released EGGROLL implementation divides all member scores by one standard deviation computed from the same population. We test the resulting finite-population effect with 4,096 coupled directions for each of two Qwen3-0.6B blocks, four radii from $10^{-3}$ to $8\times10^{-3}$, and three seeds. For each block, radius, and seed, the directions are partitioned into independent populations of size $N$. After fitting the best scalar multiple of the raw field, the directional residual is the norm of the remaining orthogonal component divided by the norm of the standardized field.

\begin{table}[tbp]
\centering
\papertablestyle
\caption{Comparison of mean updates before and after population score standardization, averaged over two transformer blocks, four perturbation radii, and three seeds. Cosine measures directional agreement between the raw and standardized means, with $1$ indicating identical directions. Residual is the component of the standardized mean that cannot be explained by rescaling the raw mean, expressed as a percentage of the standardized mean's norm. Higher cosine and lower residual indicate a smaller directional effect.}
\label{tab:normalizationaudit}
\begin{tabular*}{\textwidth}{@{\extracolsep{\fill}}rrrrr@{}}
\toprule
$N$ & Dense cosine & Dense residual & Rank-one cosine & Rank-one residual \\
\midrule
32  & $.99954$ & $3.03\%$ & $.99858$ & $5.28\%$ \\
64  & $.99979$ & $2.04\%$ & $.99924$ & $3.87\%$ \\
128 & $.99990$ & $1.44\%$ & $.99970$ & $2.42\%$ \\
\bottomrule
\end{tabular*}
\end{table}

\Cref{tab:normalizationaudit} shows that the normalized and raw fields remain closely aligned, and that their difference decreases with population size. At the training value $N=128$, the rank-one cosine is $0.99970$ with a $2.42\%$ directional residual. The dense control gives $0.99990$ and $1.44\%$. Fits across ranks $1,2,4,8$ do not produce a coefficient whose direction and magnitude remain stable under the predicted $\sigma^2$ scaling, leaving that correction below the resolution of this audit.

\end{document}